\documentclass[sn-mathphys]{sn-jnl}

\usepackage{array}
  \usepackage[caption=false,font=footnotesize]{subfig}
\usepackage{fixltx2e}
\usepackage{stfloats}

\jyear{2022}%

\theoremstyle{thmstyleone}%
\newtheorem{theorem}{Theorem}
\theoremstyle{thmstyletwo}%
\newtheorem{remark}{Remark}%

\theoremstyle{thmstylethree}%
\usepackage{comment}
\begin{document}

\title[Article Title]{Modular Multi-Copter Structure Control for Cooperative Aerial Cargo Transportation}


\author*[1,2]{\fnm{Dimitris} \sur{Chaikalis}}\email{dimitris.chaikalis@nyu.edu}

\author[2]{\fnm{Nikolaos} \sur{Evangeliou}}\email{nikolaos.evangeliou@nyu.edu}

\author[2,3]{\fnm{Anthony} \sur{Tzes}}\email{anthony.tzes@nyu.edu}

\author[1,3]{\fnm{Farshad} \sur{Khorrami}}\email{khorrami@nyu.edu}

\affil*[1]{\orgdiv{Electrical \& Computer Engineering}, \orgname{New York University}, \orgaddress{\city{Brooklyn}, \postcode{11201}, \state{New York}, \country{USA}}}

\affil[2]{\orgdiv{Electrical Engineering}, \orgname{New York University Abu Dhabi}, \orgaddress{\city{Saadiyat Island}, \postcode{129188}, \state{Abu Dhabi}, \country{UAE}}}

\affil[3]{\orgdiv{Center for Artificial Intelligence and Robotics}, \orgname{New York University Abu Dhabi}, \orgaddress{\city{Saadiyat Island}, \postcode{129188}, \state{Abu Dhabi}, \country{UAE}}}



\abstract{The control problem of a multi-copter swarm, mechanically coupled through a modular lattice structure of connecting rods, is considered in this article. The system's structural elasticity 
is considered in deriving the system's dynamics. The devised controller is robust against the induced flexibilities, while an inherent adaptation scheme allows for the control of asymmetrical configurations and the transportation of unknown payloads.
. 
Certain optimization metrics are introduced for solving the individual agent thrust allocation problem while achieving maximum system flight time, resulting in a platform-independent control implementation. Experimental studies are offered to illustrate the efficiency of the suggested controller under typical flight conditions, increased rod elasticities and payload transportation.}

\maketitle
\clearpage
\section{Introduction}
Aerial systems interacting with their environment~\cite{GT21} 
have attracted increased interest. Advancements in sensing along with on board computational strength in these aerial agents allow their usage in tasks with increased complexity. Collaborating multi-agent systems with certain constraints in their formation have also been utilized~\cite{crazyswarm}. In these cases, proper flight behavior in the presence of any physical connection between individual aerial members is an important attribute~\cite{CZW2019}.

Emerging markets using physically connected aerial vehicles are prominent in the area of aerial/drone delivery systems
, where physically connected aerial vehicles can be used to lift and transport objects of varied weights and dimensions that can not be handled by a single vehicle. The 
logistics behind drone package delivery can be found in~\cite{delivery_logistics}.

Aerial manipulators have also been used for cooperative manipulation and transportation~\cite{MFK2011,LKK2016
}. In general, advanced aerial manipulation concepts and designs bear significant similarity to systems of rigidly connected aerial vehicles due to the dynamic coupling between autonomous subsystems. In~\cite{
PNMC2020}, methods for collaboratively manipulating objects with autonomous aerial agents are studied. Similarly, \cite{LKK2015
} tackle the problem of transportation using aerial vehicles equipped with dexterous robotic arms while~\cite{Aeroarms} includes extensive work on the subject of aerial manipulators heavily interacting with the environment.

Cooperative aerial transportation via use of cables has been heavily researched, due to the added benefit of minimization of interaction forces~\cite{MS2021,RTCSL2019}. Concept designs of aerial vehicles collaboratively transporting objects via contact forces have also been studied in simulation~\cite{UEM2020
}.

In~\cite{OD2011} modules of independently powered connected rotors are considered; these rotors are attached directly on objects in~\cite{MC2019} and an estimation scheme is provided for computing the rotors' positions and thus the moment areas, while the controller uses a pseudo-inverse thrust concept in it structure. The controller requires the utilization of individual IMUs placed at each rotor as well as at the transported payload.

Single Degree-of-Freedom (DoF) connected UAVs appeared in~\cite{CZW2019} and simulation studies validated the overall concept. In~\cite{SGLYK2018,gabrich2020modquad} magnets were used for the copter-attachments followed by a single DoF joint for estimating the system's yaw. The suggested scheme assumed copters connected in sequential manner while complete knowledge of individual altitudes are required. Similarly, in~\cite{MATA2019} copters encountered in close proximity were attached relying on estimation schemes for computing the moment areas, while neural network controllers were tested. Advanced modules of rotor-wheel combinations were considered in~\cite{OJ2021}; these were attached to the carried objects and their optimal attachment positions were computed.

Overactuated aerial vehicles were considered in~\cite{SYGRT2021} and the thrust allocation problem was examined using a quad-copter on a dual-axes gimbal configuration followed by singularity avoidance methods.

The main contributions of this paper are: 1) the design using primitive components (rods, polygons, attachment mechanisms) for interconnecting in a mechanical manner copters for load transportation, 2) derivation of dynamics for the multi-copter configuration while handling structural flexibilities, 3) the design of an adaptive controller robust against these flexibilities and handling the unknown payload properties, and 4) an optimization mechanism for adjusting the thrust to each agent using apriori defined criteria (i.e., maximizing flight time). Compared to its earlier shortened version~\cite{CETK22_icuas}, this article extends the presented theoretical and experimental results by considering the effects of flexibilities in the dynamics and controllers,
while carrying a payload with unknown properties (mass, moments of inertia), as well as improved optimization metrics and experimental validations.
This paper is structured in the following manner. 
Section~\ref{sec:copter_lattice} provides the kinematic analysis of the rigid-structure of the modular copter-structure. Section~\ref{sec:dynamics_section} presents the dynamics of the modular-copter including the flexibility effects, while Section~\ref{sec:control_sec} refers to the adaptive controller design. Experimental studies are presented in Section~\ref{sec:experimental_studies} followed by Concluding remarks.
\section{Rigid Multi-Copter Kinematic Analysis\label{sec:copter_lattice}}
The kinematics of a zero-carried payload $\left(m_p=0\right)$ multi-copter structure is considered. 
Assume a planar\footnote{The copters can be placed in parallel planes and the requirement for a planar configuration can thus be relaxed.} modular lattice comprising of $n$ copters, $r,~(r \geq n)$ connecting rods and $p$ polygons, as shown in Figure~\ref{fig:copter_lattice}. Each side of the rod is connected to a copter or polygon and its other side to another polygon; the rods cannot be deformed and are assumed rigid in this section. 

Let $C_E$ be the Earth-fixed frame, $C_s$ the rigid structure's coordinate frame and $C_i$ the individual copter's coordinate system; the common attribute between $C_s$ and all $C_i,~i=0,\ldots,n-1$ is that their $z$-axes are parallel $(z_s \parallel z_0,\ldots,\parallel z_{n-1})$. The origin of $C_i$ is located at the center of mass of the $i$th-copter, the $x_i$-axis of the IMU is along the line that connects the polygon with the $i$th-copter while the $z_i$-axis is perpendicular to the copters' plane. The only exception being copters attached on top of polygons (3D structures), where the $x_i$-axis is aligned with that of the polygon. For notation simplicity, we assume that the $p_0$ polygon is connected to the $C_0$ copter.

Each $p_i$ equilateral-polygon has $p_{i,f_i}$ faces (i.e, $p_0~(p_1)$ is a hexagon~(square) and has  $f_0=6~(f_1=4)$ faces) with rods departing from its vertices having an inner rod angle of $-\frac{2\pi}{f_i}$. Each copter or polygon is relatively placed with respect to its connecting polygon at a vector with length $l_i$ and angle $j \frac{2 \pi}{f_i}$; let $l_{i}^{j}$ be the length of the rod between polygons $p_i$ and $p_j,i,j\in1,\ldots,p-1$.

Having defined in a systematic manner the mechanical interconnections between polygons and copters the kinematic description for all copters ${\cal C} =\left\{c_i\right\},~i=0,\ldots,n-1$, polygons ${\cal P}=\left\{p_j\right\},~j=0,\ldots,p-1$, types of polygons ${\cal F} = \left\{f_0,\ldots,f_{p-1}\right\}$ and rods ${\cal L}=\left\{l_k \union l_q^r\right\},~k=0,\ldots,n-1,q\neq r \in \left\{0,\ldots,p-1\right\}$ can be expressed with respect to the $\left[C_0,(x_0y_0z_0)\right]$ coordinate system. As an example, from Figure~\ref{fig:copter_lattice},
$\left[C_1,(x_1y_1z_1)\right]=\left[C_0,(x_0y_0z_0)\right] \times \mbox{Trans}(x,l_0)$$\times\mbox{Rot}(z,60^{\circ})\times\mbox{Trans}(x,-)$, and 
$\left[C_2,(x_2y_2z_2)\right]=\left[C_0,(x_0y_0z_0)\right] \times \mbox{Trans}(x,l_0)$$\times\mbox{Rot}(z,120^{\circ})\times\mbox{Trans}(x,-l_0^1)\times\mbox{Trans}(x,-l_2)$

The needed information to be stored for this kinematic analysis are the rods' length array ${\cal L}$, the polygon list ${\cal P}\times{\cal F}$ and the interconnection list for rods ${\cal I}$ corresponding to a matrix $I_{(p)\times(n+p)}$ with elements $I_{i,j},~\sum_{j=1}^{n+p}I_{i,j} =1$, where $I_{i,j} = \left\{ \begin{tabular}{cl} 1 & if $\exists$ rod between polygon $i$ and copter/polygon $j$\\0 & otherwise\end{tabular} \right.$. Inhere, the enumeration of rods regarding the polygons follows that of the copters; for the multi-copter shown in Figure~\ref{fig:copter_lattice}, 
$
{\cal I} = \left[ \begin{array}{cccccc|cc} 
1 & 1 & 0 & 0 & 1 & 1 & 0 & 1\\
0 & 0 & 1 & 1 & 0 & 0 & 1 & 0
\end{array} \right] _{2 \times 8}.
$

\begin{figure}[htbp] 
    \centering
    \includegraphics[width=0.8\columnwidth]{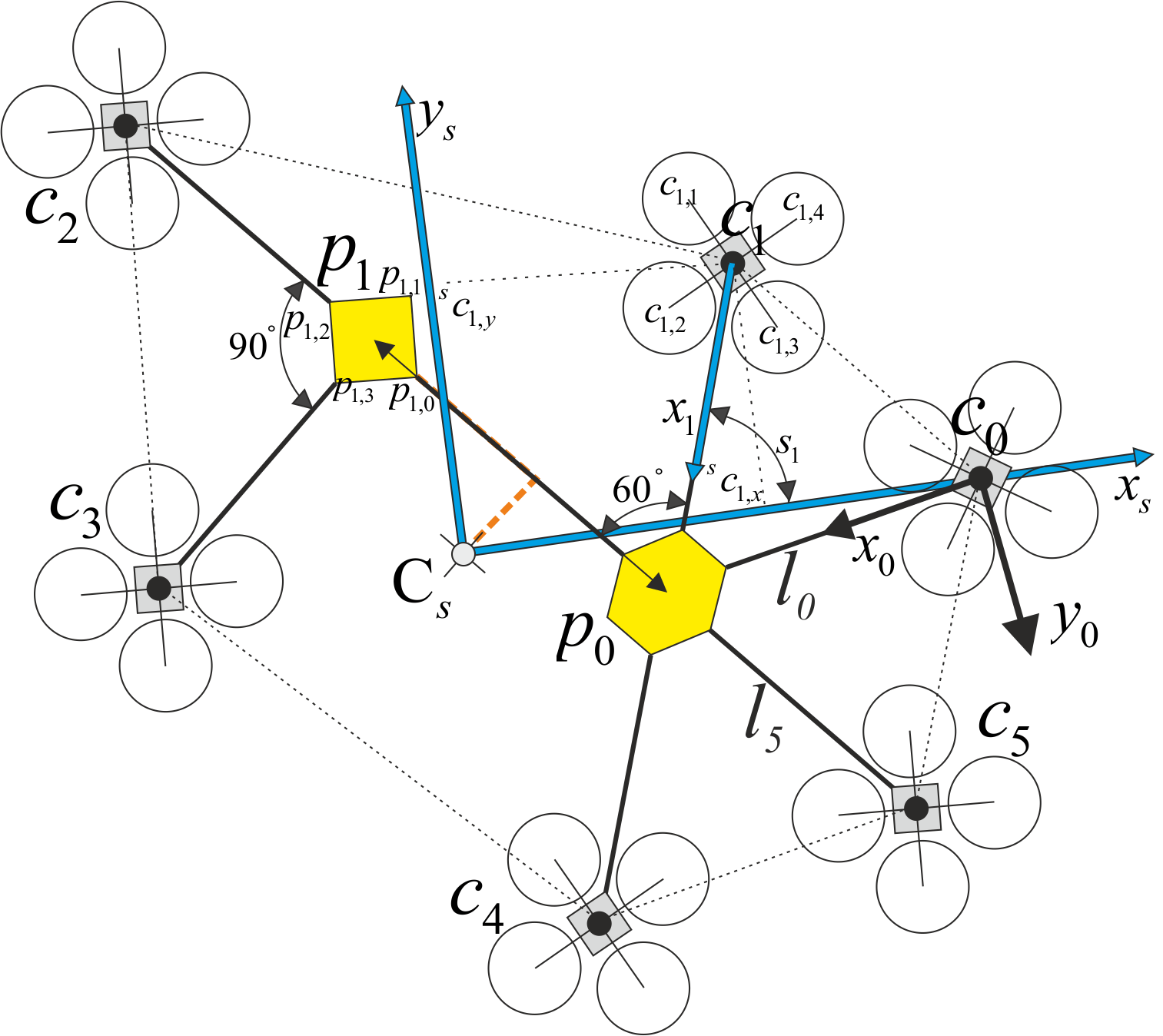}
    \caption{Copter-Lattice with $(n=6)$~copters, $(p=2)$-polygons}
    \label{fig:copter_lattice}
\end{figure}

Let the coordinate frame of the multi-copter system $\left[C_s,(x_sy_sz_s)\right]$; its $C_s$-center, $C_s \in~\mbox{Co}\left(\mathcal{C}, \mathcal{P}\right)$ coincides with the center of mass of the entire copter/polygon/rod structure. The $x_s$-axis points to the vertex $c_i \in {\cal C}$ that has the largest distance from $C_s$; in Figure~\ref{fig:copter_lattice} $\parallel \overrightarrow{C_s c_0}\parallel =\max_{i \in {\cal C}} \parallel \overrightarrow{C_s c_i}\parallel$. 
This uniform framework of assigning coordinate frames for the collective structure is necessary in order to simplify the overall development of the presented controllers.
It should be noticed for the case of carried payload $m_p \neq 0$, $\left [C_s,(x_sy_sz_s)\right]$ changes and the adaptive algorithm estimates recursively this coordinate system. 

Having assigned $\left[C_s,(x_sy_sz_s)\right]$, the pose of all copters with respect to it, can be computed. For brevity let the pose attributes for the $i$th drone be 
$\left[\left(^sc_{i,x} {}^sc_{i,y} 0\right),\alpha_i \right]$, resulting in a homogeneous matrix $^sA_i$ describing its pose with respect to $\left[C_s,(x_sy_sz_s)\right]$ 
\[
\small
    ^sA_i
    \hspace*{-1mm} 
    =
    \hspace*{-1mm}
    \left[ 
    \begin{array}{ccc|c}
\cos(\alpha_i) & -\sin(\alpha_i) & 0 & ^sc_{i,x} \\
\sin(\alpha_i) & \cos(\alpha_i)  & 0 & ^sc_{i,y} \\
0         & 0          & 1 & 0\\ \hline
0 & 0 & 0 & 1
    \end{array}
\right]
\hspace*{-1mm}
=
\hspace*{-1mm}
\left[ 
    \begin{array}{c|c}
    \mbox{Rot}\left( z,\alpha_i \right) & \begin{array}{@{}c@{}} ^sc_{i,x} \\ ^sc_{i,y} \\ 0\end{array} \\ \hline
    0_{1 \times 3} & 1
    \end{array}
\right]
\]
\normalsize
where $\alpha_i$ is the angle between the the copter and the the $x_s$-axis and $(^sc_{i,x}, ^sc_{i,y})$ is the planar\footnote{The same analysis holds for non-planar drones, where there is a nonzero displacement in the $z_s$-axis, resulting in the (3,4) element for matrix $^sA_i$ to be $^sc_{i,z}$ rather than 0.}displacement from $C_s$. 

\begin{figure}[htbp] 
    \centering
    \includegraphics[keepaspectratio,width=0.75\columnwidth]{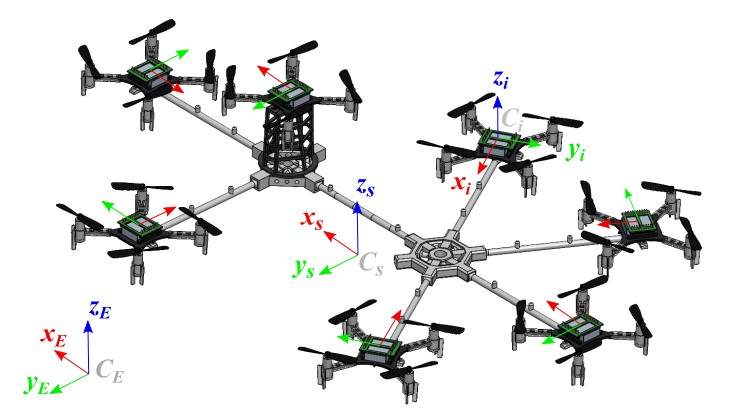}
    \caption{Example modular non-planar copter-structure.}
    \label{fig:copter_lattice}
\end{figure}

The procedure initially computes the drones' pose with respect to $\left[C_0,(x_0y_0z_0)\right]$ followed by a transformation to $\left[C_s,(x_sy_sz_s)\right]$.

The displacement vectors of all copters are required to solve the control allocation problem, while the orientation angles $\alpha_i$ are used by the attitude controller and readjusting the IMU-readings of the coupled copters.
\section{Multi-Copter Structure Dynamic Model \label{sec:dynamics_section}}
During the aforementioned kinematic analysis of the multi-copter system, it was assumed that the rods were rigid. If elongated rods are used, this gives rise to certain bending flexibilities. These flexibilities can influence the flight dynamics and counteract the benefits enjoyed by the adoption of long rods. These benefits include less required effort per agent in generating torques for maneuvering the structure, Assuming the rod-model to correspond to that of a clamped-free (or clamped-clamped) beam where the clamped side is at the used polygon and the free one at the copter side, this provides significant static displacement at its tip (copter-side). As an example, in Figure~\ref{fig:bent-frame}, a rod with the same attributes employed in our experimental studies (280mm length, 5mm diameter, ABS-material) is subjected to a 0.46~N thrust force by the attached copter at the end and results in a 10~mm bending at the copter side and an angle $\gamma \simeq$6.5~degrees; the generated result is computed using the ComSol Multiphysics software. 

\begin{figure}[htbp]
    \centering
    \includegraphics[keepaspectratio,height=0.4\columnwidth]{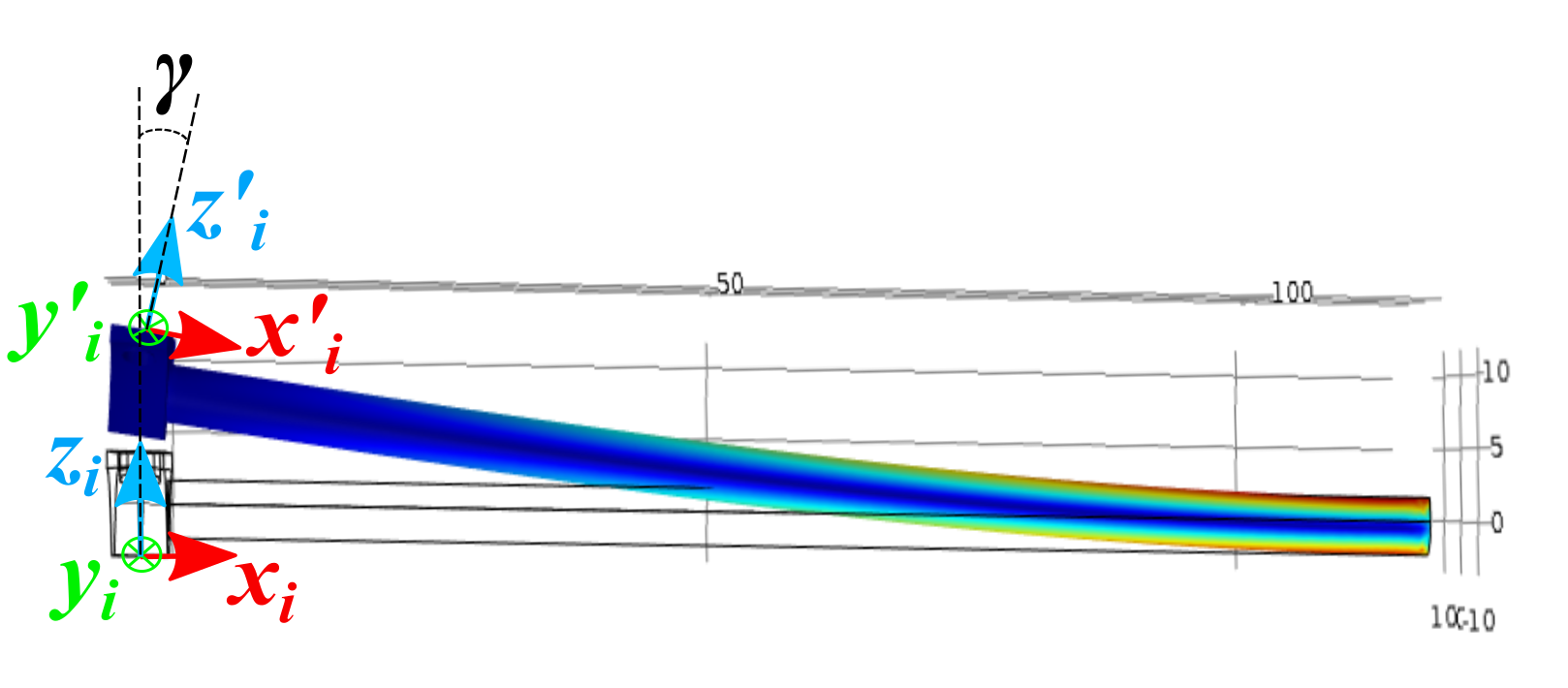}
    \caption{Copter-Polygon (Left-Right side) Flexible Bending}
    \label{fig:bent-frame}
\end{figure}

This flexibility has adverse effects on the IMU-sensor readings of each copter, since these are being made with respect to the deformed (flexible) system $\left[C_i^{'},(x_i^{'}y_iz_i^{'})\right]$ rather than the rigid one $\left[C_i,(x_iy_iz_i)\right]$. Furthermore the control inputs $T_i, M_i$ of the $i$th copter must be computed with respect to the $z_i^{'}$-axis. In the static configuration, the rotation between $C_i$ and $C^{'}_i$ is a rotation around the $y_i$-axis by a bending angle $\gamma_i$, while the relative displacement is modeled by an elevation $\delta_i^z$ along the the $z_i$ axis.
The transformation between $C_i$ and $C^{'}_i$ is
$
\mbox{Rot}_{y}(\gamma_i) \mbox{Trans}\left(z,\delta^z_i\right).
$
In the sequel let the altitude and attitude of the centroid of the multi-copter system (including the carried payload) expressed with respect to the Earth coordinate system be denoted as $^E\mathbf{X}_c=\left[ x_c,y_c,z_c\right]^T$ and $^E\mathbf{\Gamma}_c=\left[\phi,\theta,\psi\right]^T$. It should be noted that the coordinate system $\left[ C_s(x_sy_sz_s) \right]$ does not coincide with the coordinate system of the multicopter system, even when the roll, pitch and yaw angles are zero $\left[\phi,\theta,\psi\right]=[0,0,0]$ for the case of a non zero assymetric payload $(m_p \neq 0)$.

The system's dynamics can be computed by taking the extra transformation into account as:
\begin{eqnarray}
\left[
\begin{matrix}
\ddot{x}_c \\ \ddot{y}_c \\ \ddot{z}_c
\end{matrix} \right] &=& 
\left[ 
\begin{matrix}
0 \\ 0 \\ -g
\end{matrix} \right] + \frac{1}{m} {}^E \mbox{R}_s(\phi,\theta,\psi) \sum_{j=0}^{n-1} \left( \mbox{Rot}_{z}(\alpha_j) \mbox{Rot}_{y}(-\gamma_j) \left[ \begin{matrix}
0 \\ 0 \\ T_j
\end{matrix}\right] \right)\label{eq:flex_pos},
\\
\left[ \begin{matrix}
I_{x} \ddot{\phi} \\ I_{y} \ddot{\theta} \\ I_{z} \ddot{\psi}
\end{matrix} \right] &=&
- \left[ \begin{matrix}
\dot{\theta} \\ \dot{\theta} \\ \dot{\psi} 
\end{matrix} \right] \times \left[ \begin{matrix}
I_{x} \dot{\phi} \\ I_{y} \dot{\theta} \label{eq:flex_rot}\\ I_{z} \dot{\psi} 
\end{matrix} \right] +  \left[ \begin{matrix}
\tau^c_x \\ \tau^c_y \\ \tau^c_z 
\end{matrix} \right] + \left[ \begin{matrix}
\tau_x^s \\ \tau_y^s \\ \tau_z^s
\end{matrix} \right], 
\end{eqnarray}
with the torque control definition:
\begin{eqnarray}
\left[ \begin{matrix}
\tau^c_x \\ \tau^c_y \\ \tau^c_z
\end{matrix} \right] &=&
\sum_{j=0}^{n-1} \begin{small}\left( \left[ 
\begin{matrix}
^sc_{j,x} \\ ^sc_{j,y} \\ 0
\end{matrix}\right] + \mbox{Rot}_z(\alpha_j)\left[ \begin{matrix}
0\\0\\ \delta^z_j
\end{matrix} \right] \right) \times \left( \mbox{Rot}_z(\alpha_j) \mbox{Rot}_y(-\gamma_j) \left[ \begin{matrix}
0\\0\\T_j
\end{matrix} \right] \right) \end{small}+\nonumber \\
&&\sum_{j=0}^{n-1} \mbox{Rot}_z(\alpha_j) \mbox{Rot}_y(-\gamma_j) \left[
\begin{matrix}
0\\0\\M_j
\end{matrix}\right], ~\mbox{and}\\\label{eq:dynamics_flex_abstract}
\left[ \begin{matrix}
\tau^s_x \\ \tau^s_y \\ \tau^s_z
\end{matrix} \right] &=& \boldsymbol{r}_p \times \left( ({}^E R_s)^T \left[ \begin{matrix}
0\\0\\-m_p g \end{matrix} \right] \right).
\end{eqnarray}
In the previous formulation, 
$g$ is the gravity, $m$ the total mass (including the unknown payload), ${}^E \mbox{R}_s = \mbox{Rot}_z(\psi) \mbox{Rot}_y(\theta) \mbox{Rot}_x(\phi)$ is the $3 \times 3$ rotation matrix from the structure-fixed frame to the Earth, $\tau_i^s~(I_{i})$ represents the unknown static torque~(the moment of inertia) acting on the structure's $i$th-axis, $i\in\{x,y,z\}$ due to a Center of Mass (CoM) displacement, $m_p, \boldsymbol{r}_p$ the unknown mass and 3D displacement vector between structure CoM and payload CoM,
$\gamma_j, \delta_j^z$ the bending angle and  elevation (deformation) experienced by copter-$j$ respectively, $T_j$ the total thrust produced by copter-$j$ and $M_j$ the total yaw moment produced by the propellers of copter-$j,~j=0,\ldots,n-1$.

In~\eqref{eq:dynamics_flex_abstract} the only torque ($M_j$) generated by each agent $j$ is about its $z_j$- axis. This is a direct outcome of the controller design, as only the thrusts of the agents are selected to control the attitude of the structure, with each agent being commanded to produce zero individual torques about its own $x_j,y_j$ axes.

Let $s(\cdot)=\sin(\cdot),~c(\cdot)=\cos(\cdot)$, then \eqref{eq:flex_pos} can be condensed  as
\begin{eqnarray}
\left[
\begin{matrix}
\ddot{x}_c \\ \ddot{y}_c \\ \ddot{z}_c
\end{matrix} \right]
\hspace*{-0.2cm}
&=& \hspace*{-0.2cm}
-g \left[ \begin{array}{c} 0\\0\\1 \end{array} \right]
+ \frac{1}{m} {}^E\mbox{R}_s 
\begin{small}\left[
\begin{array}{ccc}
-s(\gamma_0) c(\alpha_0) &  \dots & -s(\gamma_{n-1}) c(\alpha_{n-1})\\
-s(\gamma_0) s(\alpha_0) &  \dots & -s(\gamma_{n-1}) s(\alpha_{n-1})\\
c(\gamma_0) & \dots & c(\gamma_{n-1})
\end{array} \right] 
\left[ \begin{array}{c}
T_0 \\  \vdots \\ T_{n-1} \end{array}\right]\end{small},\nonumber\\
^E\mathbf{\ddot{X}}_c
\hspace*{-0.2cm}&=&\hspace*{-0.2cm}-g e_3 + \frac{1}{m} {}^E\mbox{R}_s \Psi
\left[\begin{array}{c}
T_0 \\  \vdots \\ T_{n-1}
\end{array}\right]
\label{eq:flex_pos_exp}
\end{eqnarray}
Contrary to the rigid-dynamics case, where the sum of thrusts is relevant for position control, in the ``flexible"-rod system dynamics, the individual thrusts appear in \eqref{eq:flex_pos_exp}.

Similarly the attitude dynamics can be compacted; consider $\Omega=^E\mathbf{\dot{\Gamma}}_c$, the diagonal inertia matrix $\mathbf{J}=\left[ \begin{array}{ccc} I_x & 0 & 0\\ 0 & I_y & 0\\ 0 & 0 & I_z \end{array}\right]$,  
\begin{eqnarray}
\Xi&=&\left[\begin{array}{c|c|c} \Xi_{0}& \dots& \Xi_{n-1}\end{array} \right]^T,\mbox{~where~} 
   \Xi_{i}=\left[ \begin{array}{c}
   \delta^z_i s(\gamma_i) s(\alpha_i) + {}^sc_{i,y} c(\gamma_i) \\
    -\delta^z_i s(\gamma_i) c(\alpha_i) - {}^sc_{i,x} c(\gamma_i) \\
    {}^sc_{i,y} s(\gamma_i)c(\alpha_i) - {}^sc_{i,x} s(\gamma_i)s(\alpha_i) 
    \end{array}
    \right], 
\label{eq:XI_i}
\end{eqnarray}
and the unknown static torque vector $\boldsymbol{\tau}^s=\left[ \begin{array}{c} \tau_x^s,\tau_y^s,\tau_z^s \end{array}\right]^T$, then \eqref{eq:dynamics_flex_abstract} can be rewritten as  
\begin{equation}
\mathbf{J}{\Dot{\Omega}} = -\Omega \times \left( \mathbf{J}\Omega\right) + \Xi \left[ \begin{matrix}
T_0 \\ \vdots \\ T_{n-1}
\end{matrix}\right] + \Psi \left[ \begin{matrix}
M_0 \\ \vdots \\ M_{n-1}
\end{matrix} \right] + \boldsymbol{\tau}^s . \label{eq:flex_rot_exp}
\end{equation}
It is apparent that the yaw moments of all copters are reoriented in a similar manner thus taking into advantage the adjustment of the thrusts-induced altitude control
. 

Although the reprojection of the thrusts due to the $\Xi$ matrix seems complicated, due to the small numbers involved in the first two rows of each column $\Xi_i$ since when $\gamma_i\simeq 0^{\circ}$, then 
$
   \Xi_{i}\simeq\left[ \begin{array}{c}  {}^sc_{i,y}, - {}^sc_{i,x}, 0 
    \end{array}
    \right]^T, 
$
and its effect can easily be quantified.


Under the assumption of a clamped-free beding beam model for each polygon-drone case, where the polygon-side sorresponds to the clamped-end and the drone-side the free one, then the maximum deflection and slope for the static case can be computed. Using an Euler-Bernoulli formulation~\cite{M1992} and elastic rods with uniform density and section moments of inertia, then the elastic-rod's dynamics is 
$
    \rho \frac{\partial^2 z_i(x_i,t)}{\partial t^2} + E I \frac{\partial^4 z_i(x_i,t)}{\partial x_i^4}=0,~i\in\{0,\ldots,n-1\},~x_i\in [-l_i,0]
$
where $\rho$ is the mass linear density, $E$ is the Young's modulus, $I$ the moment of inertia, and $z_i(x_i,t)$ the deflection of the rod along the $x_i$-axis. The boundary conditions are 
$z_i(-l_i,t)=0,~\frac{\partial z_i(-l_i,t)}{\partial x_i}=0,~\frac{\partial^2 z_i(0,t)}{\partial x_i^2}=0$, and $\frac{\partial^3 z_i(l_i,t)}{\partial x_i^3}=T_i(t)-m_ig$.
Assuming static loading conditions $(T_i(t)=T_i$), then 
$z_i(x_i,t)=(T_i-m_i g) \frac{\left(x_i-l_i\right)^2}{6EI}(3l_i-x_i)$, with a maximum deflection~(slope) at the drone attachment
\begin{equation}
\delta_i^z = \max z_i(l_i,t) =\left( T_i-m_i g\right) \frac{l_i^3}{3 EI},~
\gamma_i=\left(T_i-m_i g\right)\frac{l_i^2}{2EI}.
\label{eq:flexible_bending}
\end{equation}
These static `bending' estimates of each rod, given the corresponding copter's thrust will be used in the controller design as disturbances that need to be attenuated.
\section{Controller Design} \label{sec:control_sec}
The selected controller architecture uses: 1)  a position controller computing the desired total thrust force and responsible for atlitude of the structure, and 2) an attitude controller that computes the necessary torque control inputs. Inhere, an optimizer finds individual agent thrusts and yaw moments $(T_j, M_j)$, in order to produce the desired total thrust and torques.
These individual agent setpoints are communicated to each agent's on-board flight controller to adjust the angular velocities $\Omega_j^m$ of its individual motors.

\subsection{Positioning Controller}
The controller avoids creating aggressive maneuvers, and the position controller is thus computed for the multicopter's linerized dynamics around hovering~\cite{MCO2021}. Under the assumption of equal participation by each copter to the necessary thrust $(T_i^{\circ}= \frac{mg}{n})$, then the  linearization dynamics from \eqref{eq:flex_pos_exp} around $^E\mathbf{X}_c^{\circ}=\left[x_c^{\circ},y_c^{\circ},z_c^{\circ}\right]$, $(\phi,\theta, \psi)^{\circ}=(0,0,0)$, where $(\cdot)=(\cdot)^{\circ}+\Delta(\cdot)$, result in:
\begin{equation}
\left[ \begin{matrix}
^E\Delta\dot{\mathbf{X}}_c
\\ ^E\Delta\ddot{\mathbf{X}}_c
\end{matrix} \right] = \left[ \begin{matrix}
^E\Delta\dot{\mathbf{X}}_c \\ 0_{3 \times 1} 
\end{matrix} \right] + 
\left[ \begin{array}{c}
0_{3 \times 3}\\ \hline
\mathbf{A}_\omega 
\end{array}
\right]
\left[ 
\begin{matrix}
\Delta\phi \\ \Delta\theta \\ \Delta\psi
\end{matrix} \right] + 
\left[ \begin{array}{c}
0_{3 \times n}\\ \hline
\mathbf{B}_\omega
\end{array} \right]
\left[ \begin{matrix}
\Delta T_0 \\ \vdots \\ \Delta T_{n-1}
\end{matrix} \right],~\mbox{where}
\end{equation}
\begin{small}
\begin{eqnarray}
\mathbf{A}_\omega &=& g\left[ 
\begin{matrix}
0 & \frac{\sum_{i=0}^{n-1} c(\gamma_i)}{n} & \frac{\sum_{i=0}^{n-1} s(\alpha_i)s(\gamma_i)}{n}\\
-\frac{\sum_{i=0}^{n-1} c(\gamma_i)}{n} & 0 & - \frac{\sum_{i=0}^{n-1} c(\alpha_i)s(\gamma_i)}{n}\\
- \frac{\sum_{i=0}^{n-1} s(\alpha_i)s(\gamma_i)}{n} & \frac{\sum_{i=0}^{n-1} c(\alpha_i)s(\gamma_i)}{n} & 0
\end{matrix} \right], \label{eq:A_omega}
\\
\mathbf{B}_w &=& \left[ \begin{array}{c|c|c}
-\frac{c(\alpha_0)s(\gamma_0)}{m} & \ldots & -\frac{c(\alpha_{n-1})s(\gamma_{n-1})}{m} \\
-\frac{s(\alpha_0)s(\gamma_0)}{m} & \ldots & -\frac{s(\alpha_{n-1})s(\gamma_{n-1})}{m} \\
\frac{c(\gamma_0)}{m} & \ldots & \frac{c(\gamma_{n-1})}{m}
\end{array} \right]. \label{eq:B_reg}
\end{eqnarray}
For small angles $(\gamma_i\simeq 0)$, skew symmetric matrix $\mathbf{A}_{\omega}$ and matrix $\mathbf{B}_{\omega}$ degenerate to
\[
\mathbf{A}_{\omega} = g \left[ \begin{array}{ccc} 
0&1&\frac{\sum_{i=0}^{n-1 }s(\alpha_i)\gamma_i}{n}\\
*&0&-\frac{\sum_{i=0}^{n-1} c(\alpha_i) \gamma_i}{n}\\
*&*&0
\end{array}\right],~~
\mathbf{B}_{\omega} = \left[ \begin{array}{c|c|c}
-\frac{c(\alpha_0)\gamma_0}{m} & \ldots & -\frac{c(\alpha_{n-1})\gamma_{n-1}}{m} \\
-\frac{s(\alpha_0)\gamma_0}{m} & \ldots & -\frac{s(\alpha_{n-1})\gamma_{n-1}}{m} \\
\frac{1}{m} & \ldots & \frac{1}{m}
\end{array} \right].
\]
\end{small}
It should be noted that the elements of the last row of matrix $\mathbf{B}_{\omega}$ are independent of $\gamma_i$, and thus the vertical acceleration of the multicopter system can easily be computed.
In this case (small deflection angles), 
\begin{eqnarray}
\Delta \ddot{x}_c &=& g \Delta \theta +
\sum_{i=0}^{n-1} \gamma_i \left[\frac{s(\alpha_i) \Delta \psi}{n} - \frac{c(\alpha_i) \Delta T_i}{m} \right]=
g \Delta \theta +
\sum_{i=0}^{n-1} \gamma_i \xi_i^x
\label{eq:deltax}
\\
\Delta \ddot{y}_c &=& -g \Delta \phi +
\sum_{i=0}^{n-1} \gamma_i \left[ \frac{-c(\alpha_i) \Delta \psi}{n} - \frac{s(\alpha_i) \Delta T_i}{m} \right] =
- g \Delta \phi +
\sum_{i=0}^{n-1} \gamma_i \xi_i^y
\label{eq:deltay}
\\
\Delta\ddot{z}_c &=& \sum_{i=0}^{n-1} \frac{\Delta T_i}{m} + 
\sum_{i=0}^{n-1} \gamma_i \frac{-g s(\alpha_i) \Delta \phi + g c(\alpha_i) \Delta \theta}{n} 
= \sum_{i=0}^{n-1} \frac{\Delta T_i}{m} +
\sum_{i=0}^{n-1} \gamma_i \xi_i^z
\label{eq:deltaz}
\end{eqnarray}
\begin{theorem}
The backstepping PD-alike altitude controller computes 
\begin{equation}
\sum_{i=0}^{n-1} \Delta T_i= m \left( \left[ -K_{z1} - K_{z2} \right] \dot{e}_z -
\left[ 1+ K_{z1}K_{z2} \right] e_z -\sum_{i=0}^{n-1} \gamma_i \xi_i^z\right), \label{eq:final_z_control}
\end{equation}
where $e_z = z^{\circ} - z^d$, and $K_{z1},K_{z2} >0$. 
The control input including the feedforward term $T_i^{\circ}$ and the differential thrusts $\Delta T_i$ satisfying \eqref{eq:final_z_control} forces $e_z \rightarrow 0$.
\end{theorem}
\begin{proof}
Through the use of the Lyapunov function $V_z = \frac{\parallel e_z \parallel^2 + \parallel \dot{e}_z + K_{z1} e_z \parallel^2}{2}$ and application of~\eqref{eq:final_z_control}, then $\dot{V} \leq 0$, as shown in Appendix~A and the closed-loop is rendered stable.
\end{proof}
\begin{remark}
For the rigid-case, where $\gamma_i=0$, the controller degenerates to that of a PD-controller.
\end{remark}
\begin{remark}
The controller-formulation assumes knowledge of the carried payload in computing $T_i^{\circ}$ and in \eqref{eq:final_z_control}. If this is unknown, an extra term should be added and the controller needs to be modified as follows 
\begin{equation}
\dot{\hat{m}} = - \sigma \left(\dot{e}_z + K_{z1}e_z\right) \left( -\left[K_{z1} + K_{z2}\right]\dot{e}_z - \left[ 1+K_{z1}K_{z2}\right]e_z - \sum_{i=0}^{n-1} {\gamma_i} \xi_i^z \right),
\end{equation}
with $\hat{m}$ replacing the term $m$ in \eqref{eq:final_z_control} and $\sigma >0$..
\end{remark}
\begin{proof}
See Appendix B.
\end{proof}

Having computed the $\Delta T_i$-terms, the $\xi_i^x$ and $\xi_i^y$ terms can be computed in \eqref{eq:deltax} and \eqref{eq:deltay}, respectively. Let $e_x=x^{\circ}-x^d$ and $e_y=y^{\circ}-y^d$. Then
\begin{theorem}
The backstepping based PD-alike controller, adjusting the roll and pitch angles as
\begin{eqnarray}
\theta^d & =& 
\left[ -K_{x1} - K_{x2} \right] \dot{e}_x -
\left[ 1+ K_{x1}K_{x2} \right] e_x -\sum_{i=0}^{n-1} \gamma_i \xi_i^x, \label{eq:final_x_control}\\
\phi^d &=& 
\left[ -K_{y1} - K_{y2} \right] \dot{e}_y -
\left[ 1+ K_{y1}K_{y2} \right] e_x -\sum_{i=0}^{n-1} \gamma_i \xi_i^y, \label{eq:final_y_control}\\\nonumber
\end{eqnarray}
where $K_{x1}, K_{x2}, K_{y1}, K_{y2} >0$ renders the closed-loop altitude system stable yielding $e_x \rightarrow 0$ and $e_y \rightarrow 0$. 
\end{theorem}
\begin{proof}
Same as Theorem 1.
\end{proof}
\subsection{Attitude Controller}
The development of the attitude controller is subject to the following assumptions

{\bf Assumption 1} The flexibility effects on the system dynamics are small and terms like $\delta_i^z \sin(\gamma_i)$ can safely be neglected.

{\bf Assumption 2} 
The yaw-torques induced by matrix $\Psi$ in \eqref{eq:flex_pos_exp} in the $x$- and $y$-axes can be neglected and $\Psi \simeq \left[ \begin{array}{ccc} 0 & \ldots & 0\\0 & \ldots & 0\\ 1 & \ldots & 1\end{array} \right]$. Essentially this is an indication of the relative lengths of the rods compared to the negligible flexibility effects since $\delta_i^z, \gamma_i \simeq 0$.  

Given these assumptions and for small $\gamma_i$, \eqref{eq:flex_rot_exp} is transformed to:
\begin{eqnarray}
\mathbf{J}{\Dot{\Omega}} &=& -\Omega \times \left( \mathbf{J}\Omega\right) + \boldsymbol{\tau}^c
 + \boldsymbol{\tau}^s 
 = 
-\Omega \times \left( \mathbf{J}\Omega\right) + 
\left[
\begin{array}{c} \tau_x^c\\\tau_y^c\\\tau_z^c \end{array}\right]
 + \boldsymbol{\tau}^s, \label{eq:flex_rot_exp_2}\\
 &=&
 -\Omega \times \left( \mathbf{J}\Omega\right) + 
 \left[\begin{array}{c|c|c|c}
^sc_{0,y} & \ldots & ^sc_{{n-1},y} & 0\\
-{}^sc_{0,x} & \dots & -{}^sc_{{n-1},x} & 0\\ \hline
\zeta_0 & \dots & \zeta_{n-1} & 1
\end{array}  \right] 
\left[ 
\begin{array}{c}
T_0 \\ \vdots \\ T_{n-1} \\  \hline
\sum_{i=0}^{n-1} M_i
\end{array} 
\right] + \boldsymbol{\tau}^s\label{eq:control_vector_def},
\end{eqnarray}
where $\zeta_i = \gamma_i \left[^sc_{i,y} c(\alpha_i)- {}^sc_{i,x} s(\alpha_i)\right]$;
the top-two rows of~\eqref{eq:control_vector_def} is the $(x,y)$ thrust-to-torque allocation matrix. 

The objective of the attitude controller is to compute the control vector $\boldsymbol{\tau}^c$ in order to regulate the attitude dynamics.
Subsequently the thrust optimizer assigns the individual thrusts $T_i$.
while the total yaw moment is computed from~\eqref{eq:control_vector_def} as:
\begin{equation}
\sum_{i=0}^{n-1} M_i = \tau^c_{z} - \sum_{i=0}^{n-1} \zeta_i T_i.
\label{eq:total_yaw}
\end{equation}

Let the real moments of inertia and static torques acting on a given structure be constants, then the torque control vector is defined as
\begin{equation}
\boldsymbol{\tau}^c = \Omega \times \left(\mathbf{J}\cdot \Omega\right) + \hat{\mathbf{J}}\left( -K_{\phi}(\mathbf{z}_{\phi}-K_{\phi}\mathbf{e}_{\phi})-\mathbf{e}_{\phi}-K_{\omega}\mathbf{z}_{\phi}\right) - \hat{\tau}^s, \label{eq:main_control}
\end{equation}
where $\hat{\mathbf{J}}$ and $\hat{\tau}^s$ represent adaptations for the unknown moment of inertia matrix and static torque vector, $K_\phi, K_\omega$ are diagonal positive gain matrices and $\mathbf{e}_{\phi}, \mathbf{z}_{\phi}$ are error vectors defined in Appendix C.
\begin{theorem}
The controller \eqref{eq:main_control} stabilizes the system dynamics \eqref{eq:flex_rot_exp_2} for small deflections.
\end{theorem}
\begin{proof}
See Appendix C.
\end{proof}

\subsection{Individual copter Thrust Computation} 
Having computed the $\tau_x^c$ and $\tau_y^c$ and the desired total thrust 
\[
T^d \triangleq \sum_{i=0}^{n-1}T_i = \sum_{i=0}^{n-1} \left( T_i^{\circ}+\Delta T_i \right),
\]
the individual agent thrusts $T_i,~i=0,\ldots,n-1$, and yaw torques need to be provided. The thrusts are related through the allocation matrix equation
\begin{eqnarray}
\left[ \begin{matrix}
\tau^c_{x} \\ \tau^c_{y} \\ T^d
\end{matrix} \right] 
\hspace*{-1mm}&=&\hspace*{-0.2mm}
\sum_{i=0}^{n-1} \left\{
\left( \left[
\begin{matrix}
^sc_{i,x} \\ ^sc_{i,y} \\ 0
\end{matrix} \right] \times \left[
\begin{matrix}
0 \\ 0 \\ T_i
\end{matrix} \right] \right)+ \left[
\begin{matrix}
0 \\ 0 \\ T_i
\end{matrix} \right] \right\}
=
\mathbf{\Gamma} \mathbf{T}= \mathbf{\Gamma} \left[ \begin{array}{c} T_0 \\ \vdots\\T_{n-1}\end{array}\right].
\label{eq:allocation_matrix_equation}
\end{eqnarray}

For multi-copter systems, where $n > 3$, \eqref{eq:allocation_matrix_equation} has infinite solutions in computing $\mathbf{T}$. In most works, the pseudo-inverse is used~\cite{ADGS2012}. Inhere, an optimization procedure is devised for the computation of $\mathbf{T}$ while satisfying other metric.  

\subsubsection*{Optimization Procedure}
The generic optimization scheme is:
\begin{align}
&\min_{T_0,\dots, T_{n-1}} f \label{eq:optimizer}\\
\mbox{subject to: }& 
\mathbf{\Gamma} \left[ \begin{array}{c} T_0 \\ \vdots\\T_{n-1}\end{array}\right]  = \left[ \begin{matrix}
\boldsymbol{\tau}^c_{x} \\ \boldsymbol{\tau}^c_{y} \\ T^d
\end{matrix} \right],\nonumber\\
&0 \leq T_i \leq T{^{\max}},~ i = 0,\ldots,n-1,\nonumber
\end{align}
where $T^{\max}$ the maximum thrust that can be provided by the agent.
\subsubsection{Flight Time Maximization with Maneuvering Efficiency}
The system's flight time is dictated by the minimum flight time among agents; since in case of an agent's battery depletion the system will need to land. In~\cite{abeywickrama2018comprehensive}, a direct relationship between payload, thrust, and UAV-battery drainage was investigated. Hence maximizing the flight time is equivalent to minimizing the worst case of thrusts among all agents, or
\begin{equation}
f_T = \parallel \left[ T_0,\dots ,T_{n-1}\right]^{\top} \parallel_{\infty}. \label{maximum}
\end{equation}
Other approaches include a quadrant-based thrust selection procedure~\cite{SGLYK2018} and the pseudo-inverse which in general result in a smaller mean thrust value among all agents, while the adoption of the $\infty$-norm results in the lowest maximum thrust among all agents.
%

%
Reducing the system's maneuvering response is equivalent to generating large torques from the agents. Since the $T^{\max}$ is fixed, these torques can be generated by copters away from the system's center of mass~\cite{SGLYK2018,MSMK2013}. This is explained since in~\eqref{eq:allocation_matrix_equation}, copters with larger $^sc_{i,x}, ^sc_{i,y}$ generate the same torques with smaller thrust adjustment. Thus, agents with large values of $\mathbf{\Gamma}(1,i)$ and $\mathbf{\Gamma}(2,i)$ should participate more in the control effort. Since a minimization scheme is employed, then agents with small values of $\mathbf{\Gamma}(1,i)$ and $\mathbf{\Gamma}(2,i)$ should be rewarded. To further emphasize this reward, some agents with low torques are considered inactive and the ones with high torques are further emphasized; variables $\epsilon_x$ and $\epsilon_y$ are introduced for his reason, defined as
\[
\epsilon_x\left(\alpha_{\min},\alpha_{\max}\right)=
\mbox{sat}\left(  \left\lvert\frac{\tau_x^c}{\tau_{x,\max}^c} \right\rvert,\alpha_{\min},\alpha_{\max} \right),
\]
where the saturation function returns $0$ when the normalized input $\left\lvert\frac{\tau_x^c}{\tau_{x,\max}^c} \right\rvert$ drops below $\alpha_{\min}$ and $1$ above $\alpha_{\max},~(0\leq\alpha_{\min}<\alpha_{\max}\leq 1)$ and $\tau_{x,\max}^c$ is the largest anticipated control effort in the $x-$direction. A similar function is defined for the $y$-component, named $\epsilon_y\left(\alpha_{\min},\alpha_{\max}\right)$. Then the maneuvering efficiency metric encapsulating the aforementioned is 
\begin{equation}
f_{M} = \sum_{i=0}^{n-1} \left( \frac{\epsilon_x}{\lvert \mathbf{\Gamma}(1,i) \rvert} + \frac{\epsilon_y}{\lvert \mathbf{\Gamma}(2,i) \rvert} \right) T_i.\label{efficiency}
\end{equation}

A combined version of the previous metrics maximizes the flight time while reducing the response time as 
\begin{equation}
f_{E} = \epsilon f_T + (1 -\epsilon) f_M,~0\leq\epsilon\leq 1.\label{combination}
\end{equation}
This metric utilizes agents further from the centroid, while maneuvering, while optimally distributes the thrust requirements among the agents in near-hovering conditions
\subsubsection{Battery Time Optimization}
The individual agent's battery consumption~\cite{bauersfeld2022range,hwang2018practical} depends on several factors, including the rotor angular velocity, the motor-propeller combination, while frequently ignore aerodynamic effects of rotary-wing aircraft. Blade-element-momentum theory is used for describing these effects while the remaining battery capacity is calculated using Peukert model. The resulting model is quite complicated and rather empirical and results in polynomial methods to express the battery capacity. The adopted battery time metric is formed along this emprirical method and the adopted metric is 
\begin{equation}
f_{B} = \sum_{i=0}^{n-1} \frac{T_i^2}{1 - e^{(\delta - B_i)}} , \label{battery}
\end{equation}
where $B_i$ is the battery voltage reading of agent $i$. The $\delta$-constant is the lowest operational voltage of the battery and depends on the payload and the induced temperature. This metric penalizes agents with battery levels $B_i \simeq \delta$ and assumes that the battery is depleted according to the square of the thrust.

For all metrics, the YALMIP optimization toolbox
was utilized 
to provide a solution to the optimization problem. The typical optimized solution could be found in 5~msec and updated this solution (thrust and yaw moment setpoints) on the controller component every 5~msec.
%
\section{Experimental Studies\label{sec:experimental_studies}}
\subsection{Control Framework Implementation}
In order to validate the proposed controllers, the Crazyflie quadrotors~\cite{giernacki2017crazyflie} vehicles were used, in order to create flying structures
Hexagons and square connecting devices were used, while the structure rods are identical with $l_i=0.14$m. The weight of each rod is $3.5$~g, the hexagon's is $9$~g and the square's is $7$~g. Powerful neodymium magnets were inserted at each rod-end capable of providing a $7$~N attractive force.  
The Crazyflie-ROS stack~\cite{HA2017} is used for communication with all agents and the ground station. 
The control framework is shown in Figure~\ref{fig:control_diagram}.
\begin{figure}[htbp]
    \centering
    \includegraphics[width=\columnwidth]{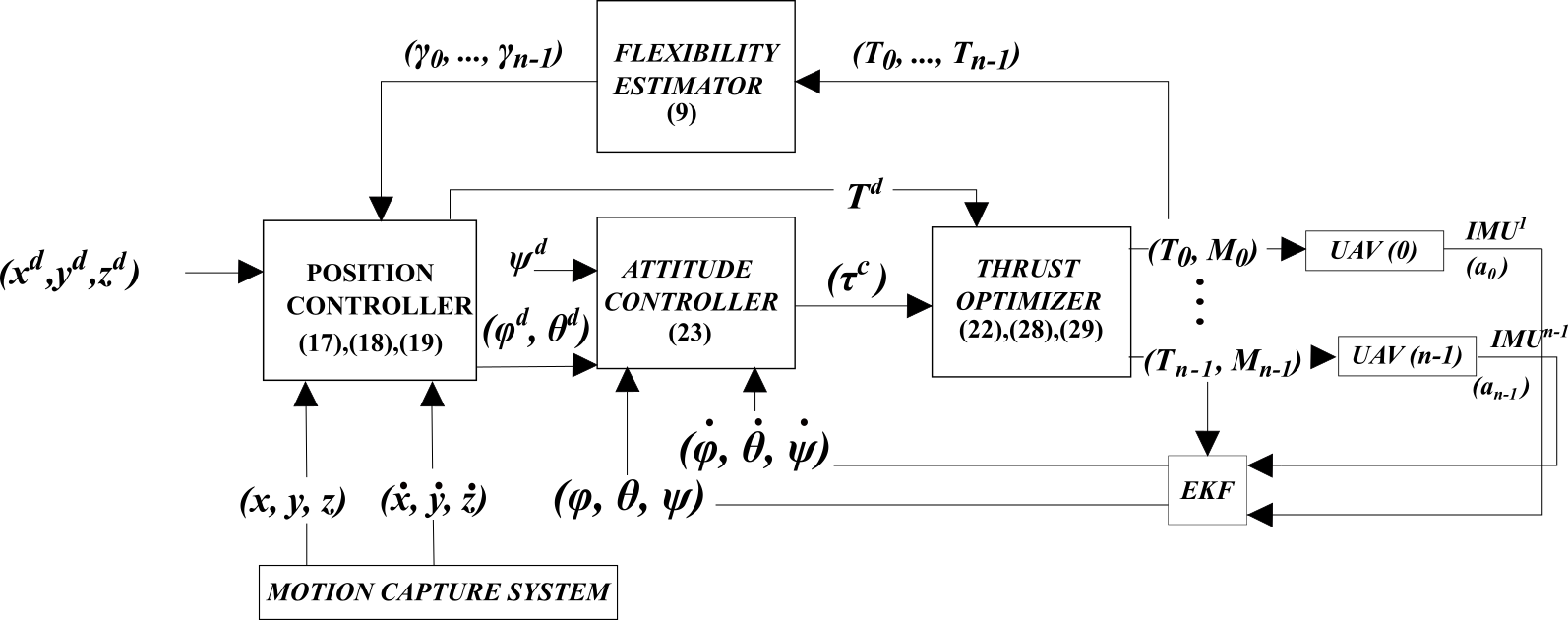}
    \caption{Controller Diagram}
    \label{fig:control_diagram}
\end{figure}

The ground station estimates the orientation of the multi-copter system via an Extended Kalman Filter (EKF).
The prediction step is performed by using the latest control commands, while the attitude update uses measurements received from all agents' IMUs~\cite{LAWCS2013
}. 
The position controller depends on position feedback from a motion capture system at a 120~Hz rate.
Once the controller has computed the attitude control and subsequently the thrust and yaw moment required for each agent
$\left[\phi_i^{d}, \theta_i^{d}, T_i, M_i\right]^{\top}$ is generated. $\phi_i^d$ and $\theta_i^d$ are extracted from $R_{z}(\alpha_i) R_{B}^W R_{z}(\alpha_i)^{\top}$, 
where $R_{B}^W$ is the rotation matrix representation of the current estimate of the multi-copter's attitude maintained at the ground station's EKF and $\alpha_i$ is defined in Section~\ref{sec:copter_lattice}. Indirectly, the rotation of agent-$i$ as seen from its own IMU is used, thus leading to commands $\phi_i^{d}, \theta_i^{d}$ equal to the roll and pitch angle estimates maintained at the on-board of the $i$th agent autopilot. 
The setpoints are transmitted to each Crazyflie every~5~msec, while the commands to the position controller are transmitted every 20~msec.
\subsection{Crazyflie Battery Characterization}
The operational battery voltage characteristics for the Crazyflie quadcopters is depicted in in Figure~\ref{fig:simulation_1}, where $B_i^{\max}=4$Volt and $\delta=2.75$~Volt for 80\% of the maximum allowable payload ($m_p^{\max}=13.125$~g) and $\delta=2.6$~Volt for $m_p=0$~g.

The battery depletion history for two Crazyflie quadcopters in a hovering mode is shown in Figure~\ref{fig:simulation_1}. In all cases, there is a sudden drop from the non-operating voltage, followed by a slow voltage drop leading to a sudden voltage drop before the need to land these quadcopters. The non-carrying quadcopter's batteries lasted until 430~seconds (blue line) in comparison with the loaded agent which lasted 270~seconds. Different $\delta$ values were recorded depending on the payload. The maximum value of all $\delta$s for varying payloads was used; $\delta=2.9$~Volt for $m_p=90$\% of the maximum available carrying payload.
\begin{figure}[htbp]
    \centering
    \hspace*{-0.6cm}
    \includegraphics[keepaspectratio,height=0.55\columnwidth]{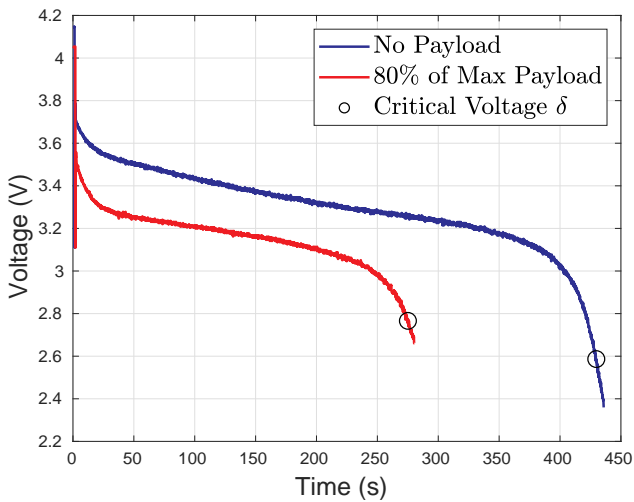}
    \caption{Battery depletion voltage $B_i$ history for various $m_p$.}
    \label{fig:simulation_1}
\end{figure}
\subsection{Flight Experiments with Rigid Copter-Structure}
In this section, experiments are presented, showcasing the efficiency of the proposed control methods in achieving cooperative flight. Cases spanning $n=3,\ldots,6$ copters are presented with short rods which correspond to a rather rigid structure, since $\max \delta_i^z=0.0025$m and $\max \gamma_i=1.37^{\circ}$.

%
\subsubsection{Quad-copter \& T-copter}
A symmetric quad-copter ($n=4$) and an asymmetric T-copter ($n=3$) were configured, as shown in Figure~\ref{fig:square-copter}, where the T-copter is created by removing one rod/copter. The same waypoints were transmitted to both 4-copter and 3-copter systems, which had the same controller employed relied on the pseudoinverse in solving \eqref{eq:allocation_matrix_equation}.
\begin{figure}[htbp]
    \centering
\includegraphics[width=0.6\columnwidth,keepaspectratio]{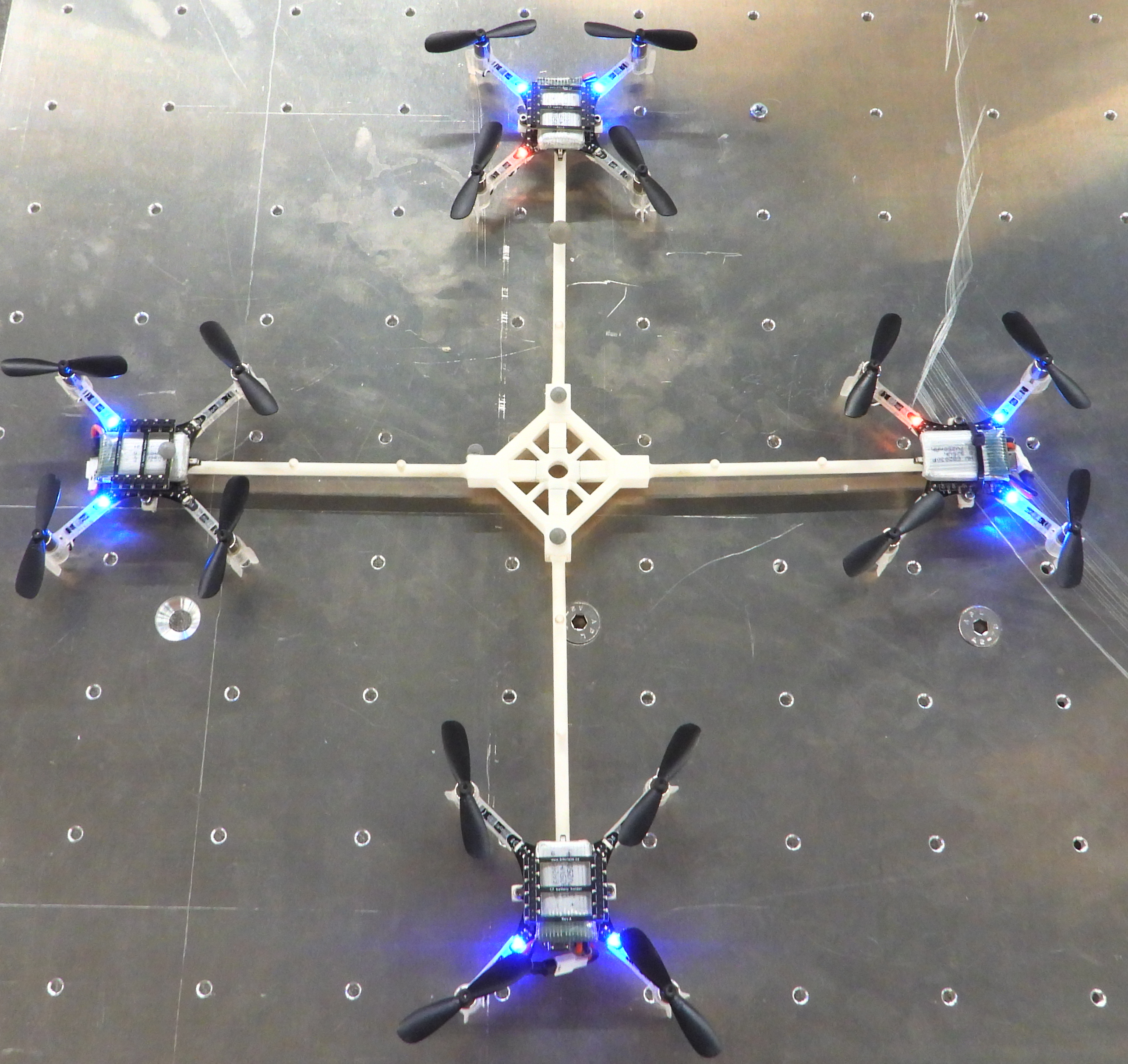}
    \caption{Modular quad-copter configuration.}
    \label{fig:square-copter}
\end{figure}

As shown in Figure~\ref{fig:quad_t-copter_response}, the quadcopter achieves slightly faster transitions and has less turbulence during its take-off phase which is anticipated due to its symmetry. 
\begin{figure}[htbp]
    \centering
    \includegraphics[width=0.8\columnwidth,keepaspectratio]{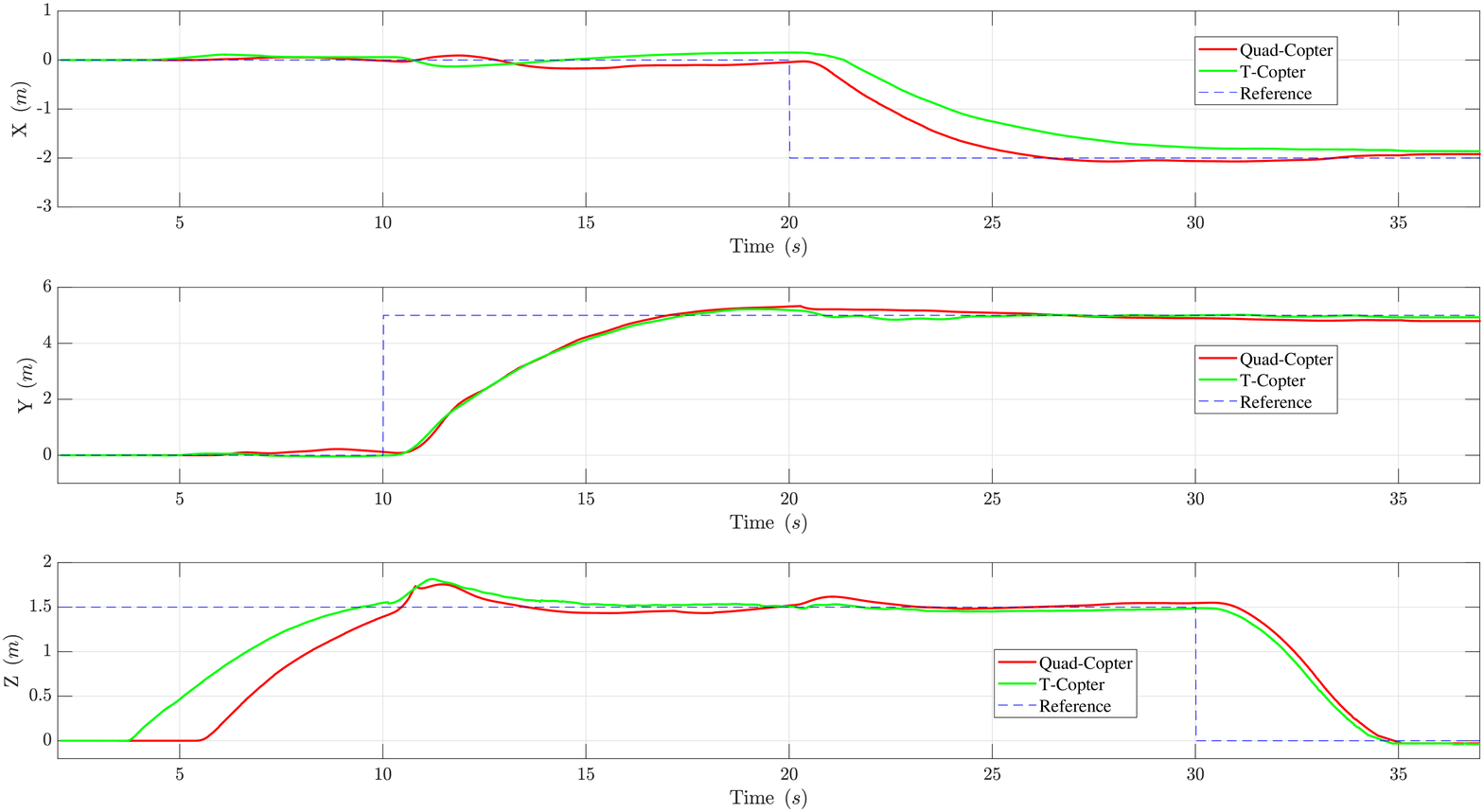}
    \caption{Quad-copter and T-copter responses}
\label{fig:quad_t-copter_response}
\end{figure}

Figure~\ref{fig:t-copter_adaptations} shows the evolution of the $x,y$ static torque and the $x,y$ diagonal elements of the inertia matrix adaptations for the T-copter configuration. In asymmetric configurations, the need to offer adaptations is apparent since there are different converging values between $\hat{J}_{xx}$ and $\hat{J}_{yy}$ and similarly between $\tau_x^s$ and $\tau_y^s$. 
\begin{figure}[htbp]
    \centering
    \includegraphics[width=0.75\columnwidth,keepaspectratio]{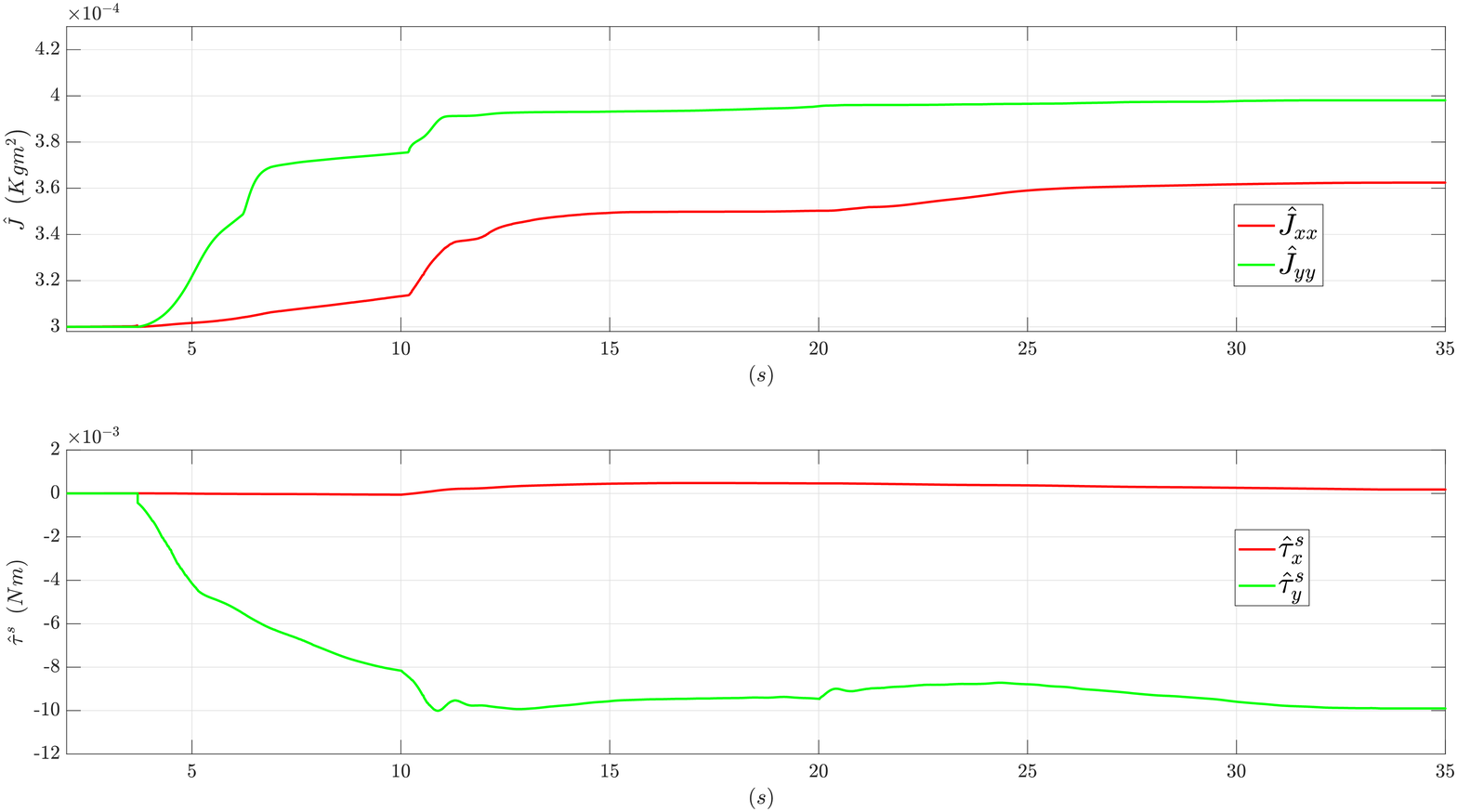}
    \caption{T-copter adaptation evolutions}
\label{fig:t-copter_adaptations}
\end{figure}
\subsubsection{Highly Asymmetric Hexacopter}
The hexacopter of Figure~\ref{fig:copter_lattice}, seen mid-flight in Figure~\ref{fig:flying_hexa}, was configured to test the controller design. The connecting rod between the hexagon to the square element exhibited significant vibrations which affected the system's overall response. The system's center of mass does not lay on any of the connecting elements (rods or polygons).

Moreover, the agents thrust exceeded $70\%$ of their maximum throttle for lifting the structure's weight, 
\begin{figure}[htbp]
    \centering
    \includegraphics[width=0.8\columnwidth]{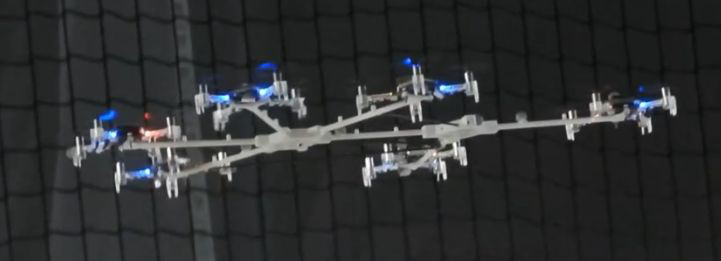}
    \caption{Assymetric Hexa-copter in Mid-Flight.}
    \label{fig:flying_hexa}
\end{figure}

For a similar waypoint navigation, proper flight can be achieved, as seen in Figure~\ref{fig:hexa_hover} (for three waypoints), showcasing the effectiveness of the proposed controller, in flying arbitrary structures.
\begin{figure}[htbp]
    \centering
    \includegraphics[keepaspectratio,width=0.8\columnwidth]{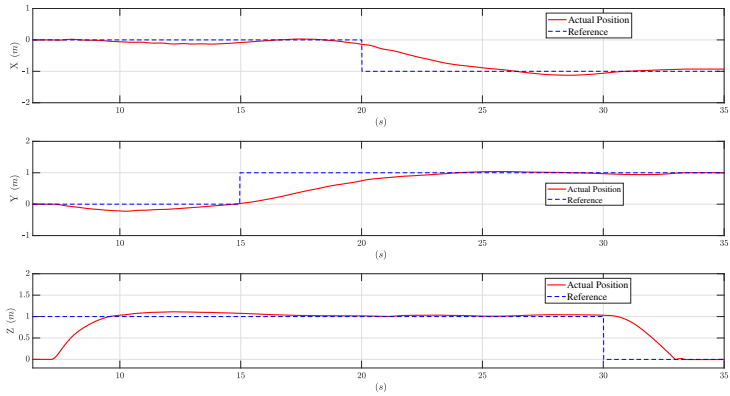}
    \caption{Path following by the hexa-copter configuration.}
    \label{fig:hexa_hover}
\end{figure}

\subsection{Thrust Allocation Comparisons}
A pentacopter shown in Figure~\ref{fig:pentacopter} was created, in order to compare the different thrust control allocation methods proposed and their preferred usage; the copter enumeration is shown in the same Figure. In the subsequent cases, the criterion of maximizing the flight time and decrease the maneuvering response time is used from~\eqref{combination}, where $\epsilon=0.67,~\alpha_{\min}=0.1,~\alpha_{\max}=1$ and $\tau_{x,\max}^c=0.09$Nm,~$\tau_{y,\max}^c=0.09$Nm.
\begin{figure}[htbp]
    \centering
    \includegraphics[width=0.8\columnwidth]{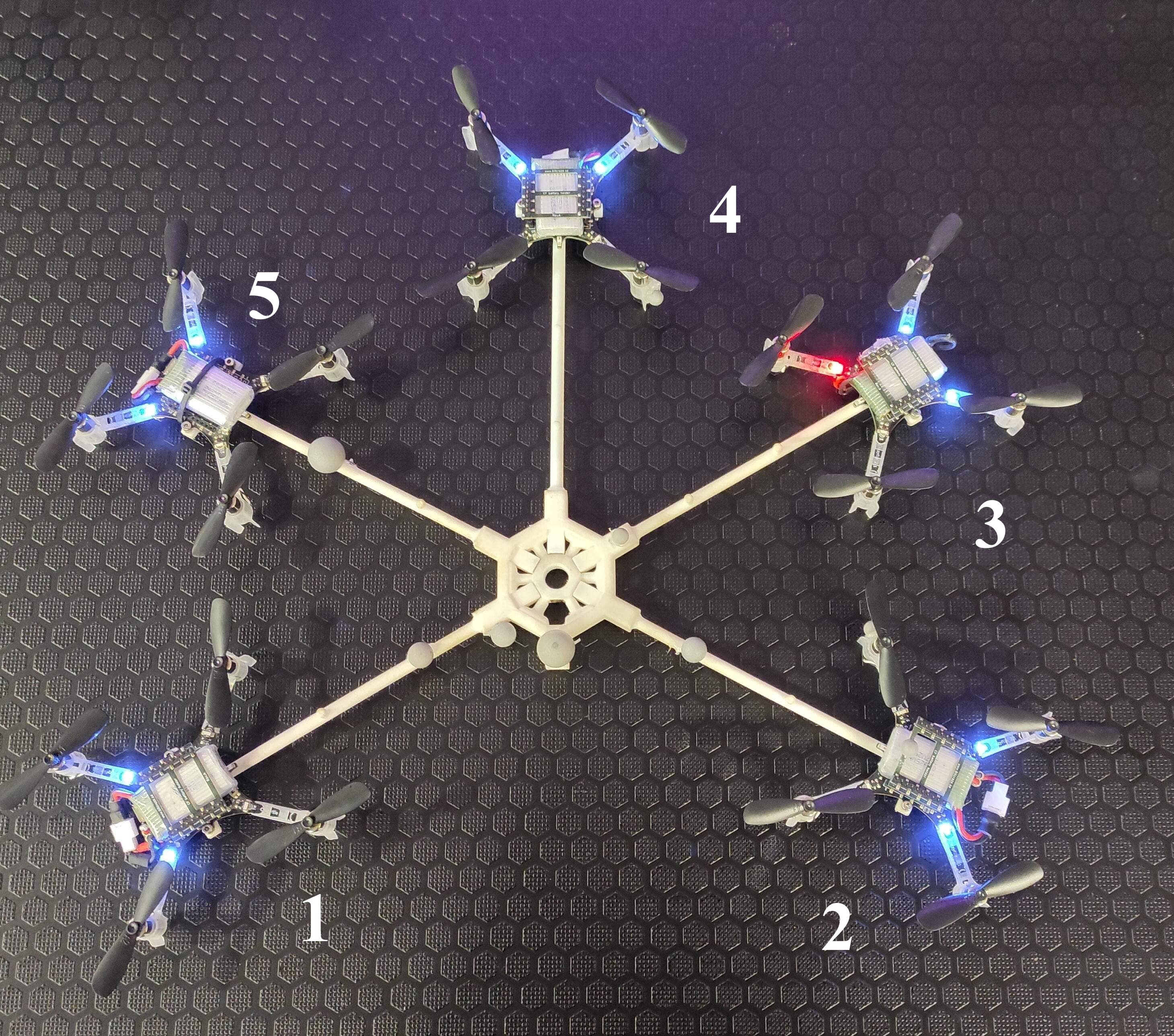}
    \caption{Assymetric Pentacopter System.}
    \label{fig:pentacopter}
\end{figure}

\subsubsection{$f_E$-metric vs. Pseudo-Inverse Response}
The first experiment was conducted to compare the performance of the thrust allocation controller of~\eqref{combination} with that of the simple pseudo-inverse. A series of step maneuvers were commanded to the pentacopter. The achieved flight history is shown in Figure~\ref{fig:penta_flights} and the corresponding thrusts commanded by the controller for each agent are seen in Figure~\ref{fig:penta_thrusts}.
\begin{figure}[htbp]
    \centering
    \includegraphics[keepaspectratio,width=0.8\columnwidth]{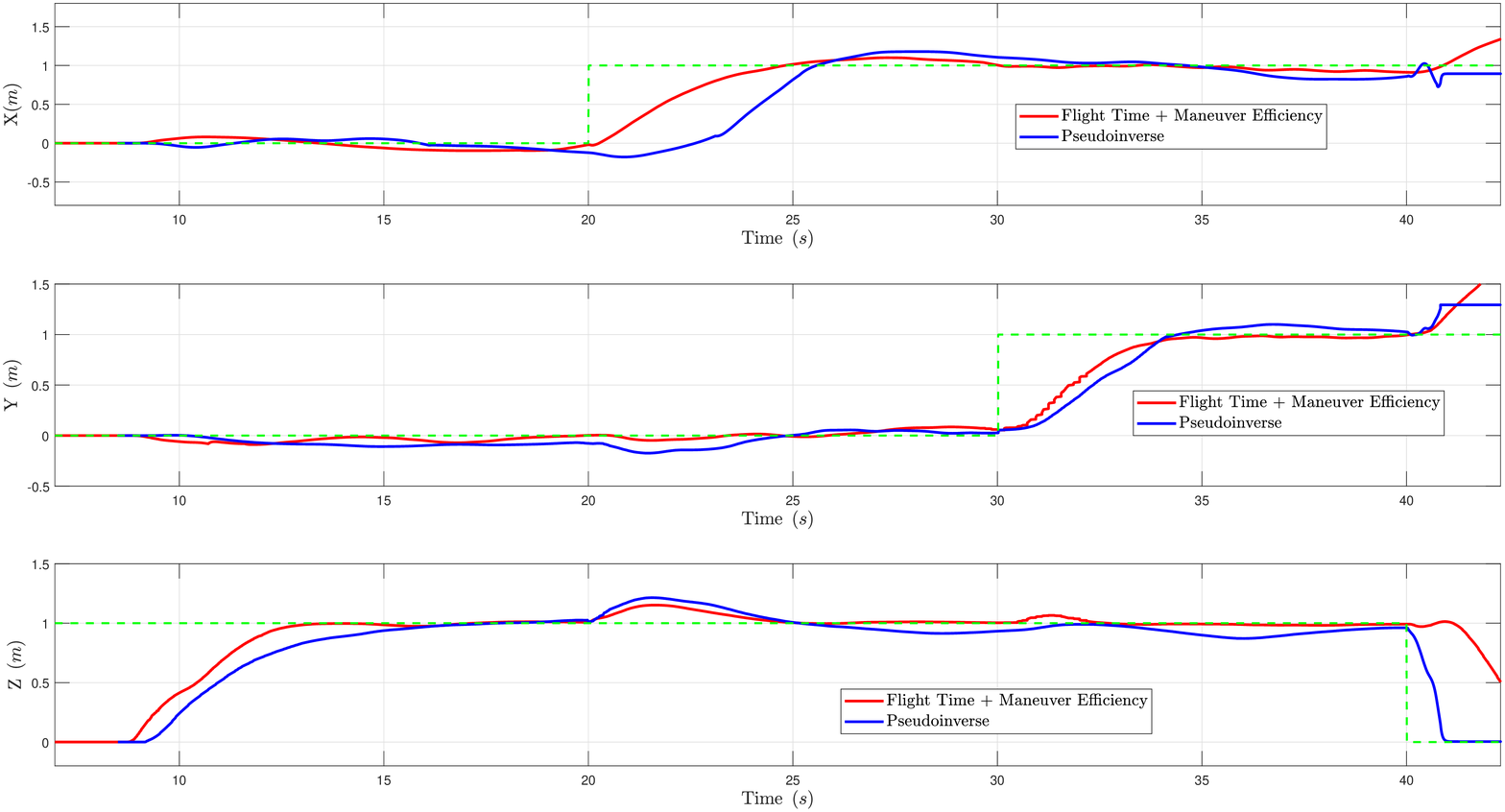} \caption{Waypoint movement comparison for the penta-copter.}
    \label{fig:penta_flights}
\end{figure}

\begin{figure}[htbp]
    \centering
    \hspace*{-0.5cm}
    \includegraphics[keepaspectratio,width=\columnwidth]{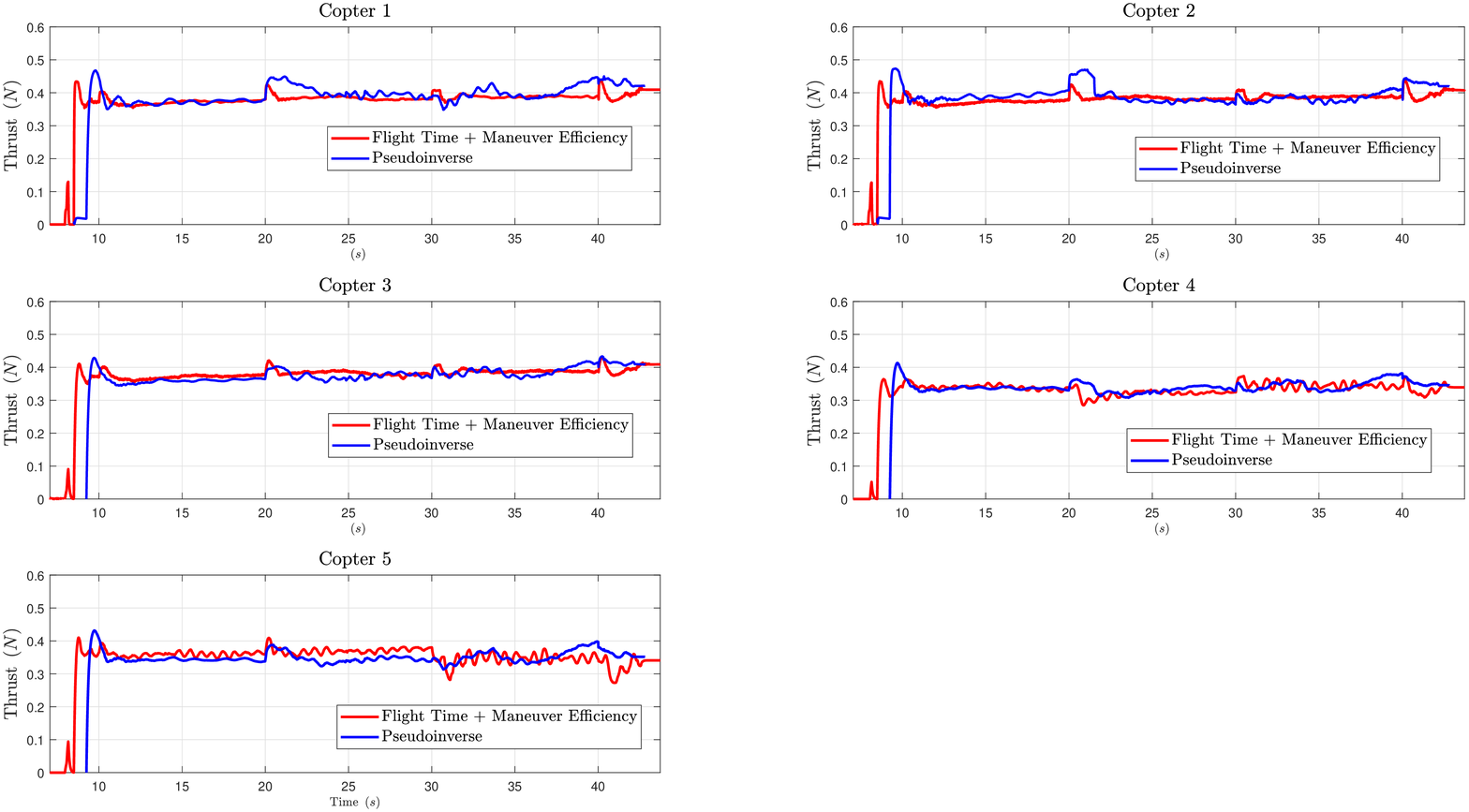} 
    \caption{Agent thrust comparison for the penta-copter.}
    \label{fig:penta_thrusts}
\end{figure}
The proposed $f_E$-minimization controller outperforms the pseudo-inverse, achieving a smoother flight with less oscillations by examining the thrust commands in Figure~\ref{fig:penta_thrusts}, looking into the high frequency thrust component of the copters. For the $f_E$-thrust allocation controller, short thrust excursions are observed during its maneuvers; these occur when $\epsilon_x$ or $\epsilon_y \neq 0$, which corresponds to 4\% of the flight time. Furthermore, the `pseudoinverse'-controller can result in infeasible commands that need to be saturated. 


%
\subsubsection{Battery-Life Optimizer}
The efficiency of the thrust allocation controller of~\eqref{battery} was examined for the penta-copter platform. Initially, all batteries were fully charged at 4.1~Volt, except for the battery on agent $3$, which was depleted at 3.85~Volt. A take-off followed by a hovering experiment was conducted to quantify the proposed optimizer's efficiency compared to the `pseudo-inverse' approach.

The control allocation shown in Figure~\ref{fig:penta_depleted_thrusts} commands on the average 10\% less total thrust from  agent-3 compared to the pseudo-inverse, due to its awareness of the battery voltage level. This is a significant reduction, given that an individual CrazyFlie agent was operating at $59\%$ of its thrust capacity in autonomous hovering.
\begin{figure}[htbp]
    \hspace*{-0.1cm}
    \includegraphics[keepaspectratio,width=\columnwidth]{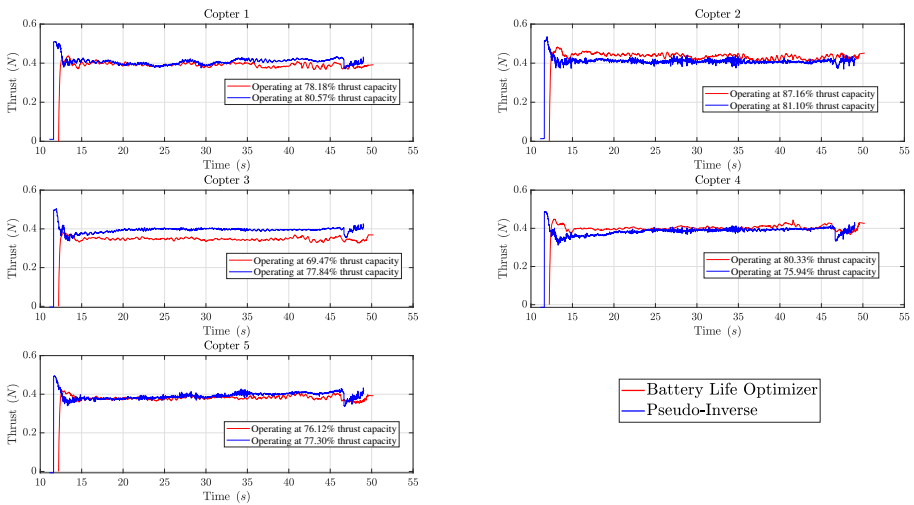}
    \caption{Thrust commands for the agents, during the battery life optimizer scenario.}
    \label{fig:penta_depleted_thrusts}
\end{figure}

The effects of the reduced thrust requirement on agent $3$ are visible in Figure~\ref{fig:penta_depleted_batteries}, where the evolution of the battery voltage readings is shown. For all other agents, except the third one, there is a faster rate of voltage decrease, resulting in small voltages $B_i,~i=1,2,4,5$ after the completion of the experiment.
\begin{figure}[htbp]
    \hspace*{-0.5cm}
    \includegraphics[keepaspectratio,width=\columnwidth]{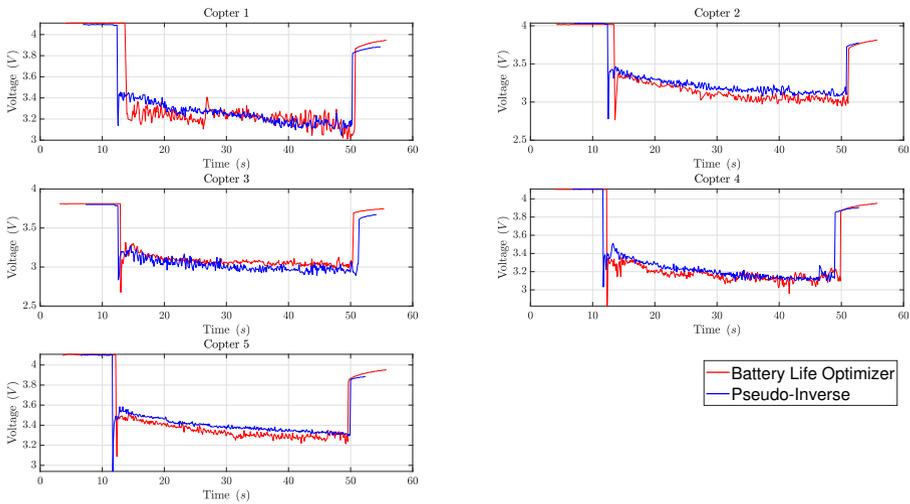}
    \caption{Battery voltages for the agents, during the battery life optimizer scenario.}
    \label{fig:penta_depleted_batteries}
\end{figure}

\subsection{Payload Transportation Experiments}
In this section, experiments are conducted to showcase the usage of the designs in cooperative payload transportation of an asymmetric payload of an L-shape.T- and L-junctions were incorporated to form the structure's skeleton and sturdily connect polygons and carbon rods in a lightweight manner, as shown in Figure~\ref{fig:styrochain}.
\begin{figure}[htbp]
    \centering
    \includegraphics[keepaspectratio,width=0.8\columnwidth]{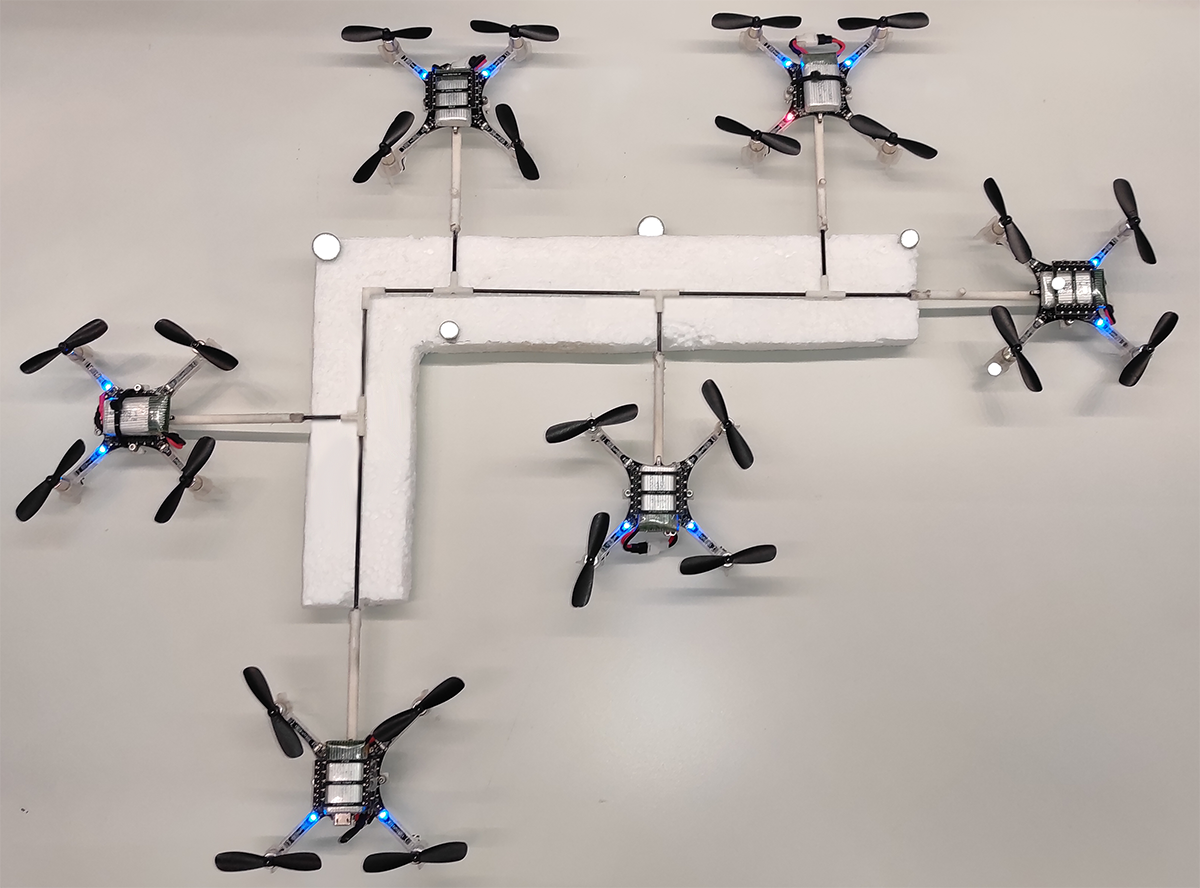}
    \caption{Collaborative payload transportation.}
    \label{fig:styrochain}
\end{figure}

Typical waypoint navigation flight segments for the structure are plotted in Figure~\ref{fig:styrofoam_flight}, along with the thrusts commanded to each agent by the thrust allocation optimization controller in Figure~\ref{fig:styrofoam_thrusts}.
\begin{figure}[htbp]
    \centering
    \includegraphics[keepaspectratio,width=0.8\columnwidth]{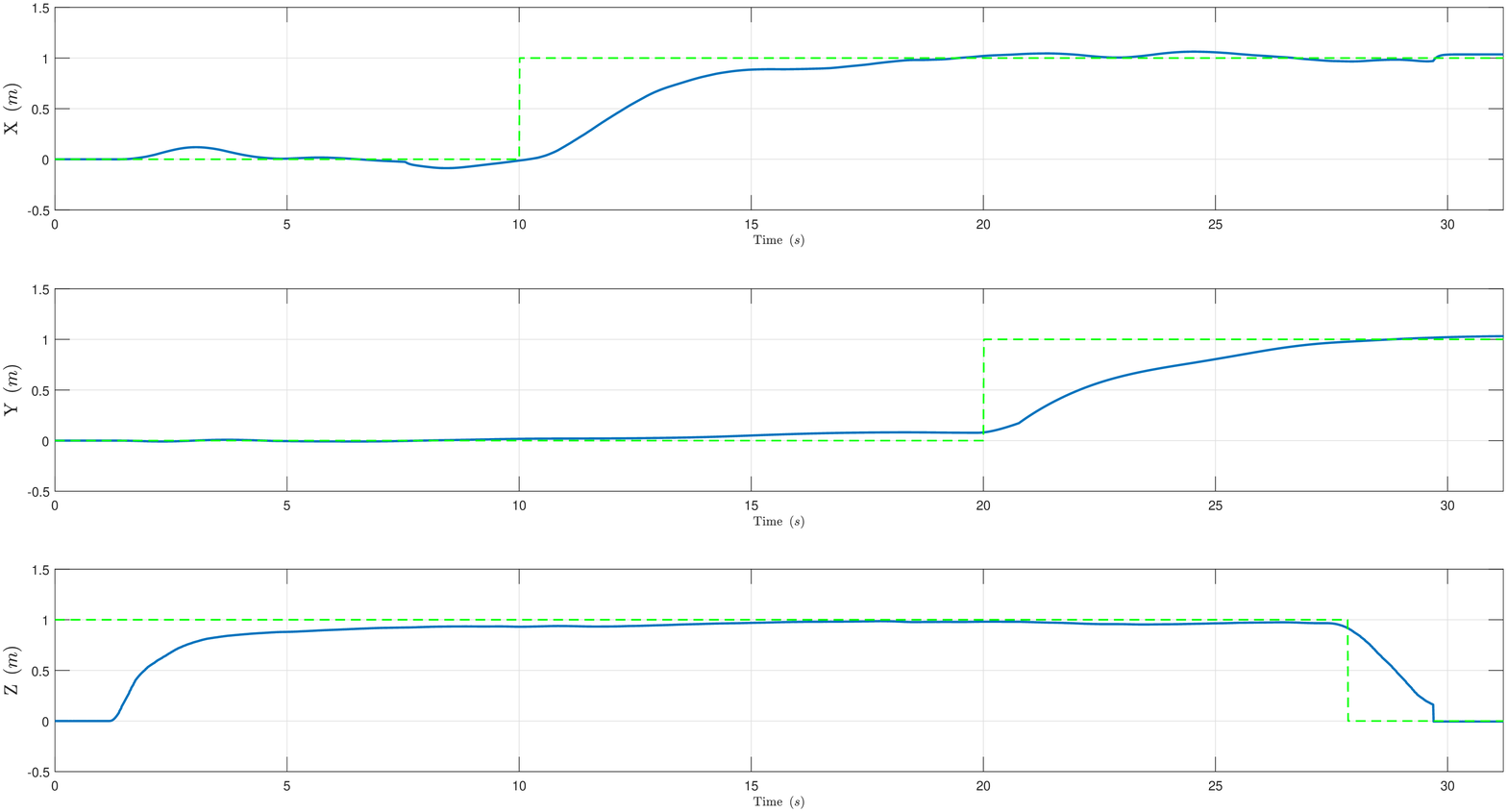}
    \caption{Collaborative payload transportation position-history.}
    \label{fig:styrofoam_flight}
\end{figure}
\begin{figure}[htbp]
    \includegraphics[keepaspectratio,width=\columnwidth]{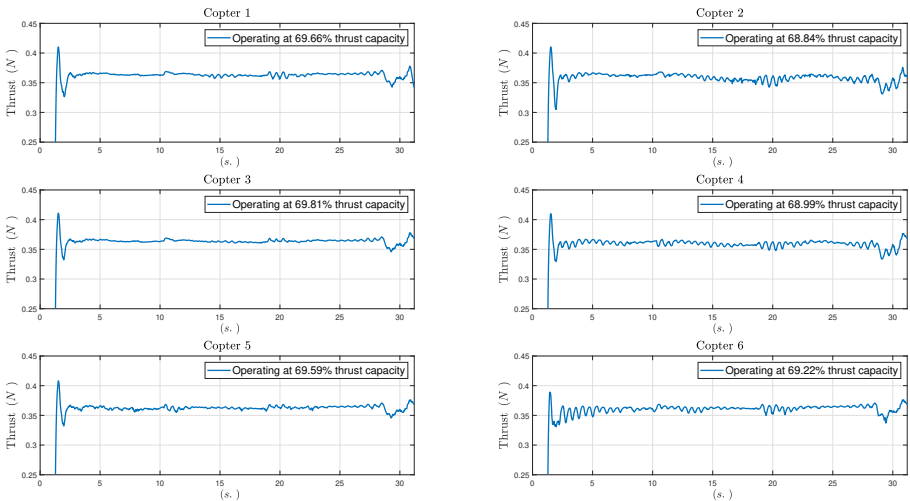}
    \caption{Agent thrusts during collaborative payload transportation.}
    \label{fig:styrofoam_thrusts}
\end{figure}
\subsection{Flight Experiments with Flexible Copter-Structure}
%
A T-copter structure with four copters ($n=4$) is used in this Section; the fourth one is at an elevated\footnote{The developed method is also valid for copters that are parallel-placed at different altitudes} centered position. Two rods of $l_1=l_2=14cm$ are used while the third one has 30~cm length. The long rod resulted in a flexible structure as shown in Figure~\ref{fig:t-copter-long-bend}. Due to the flexibility of the third rod $\delta_3^z=0.02$m and $\gamma_3=5.73^{\circ}$, while the elevated one was at 6~cm vertical distance from the remaining ones. The $f_E$-metric optimizer was used for the thrust allocation controller, while the parameter $\gamma_i$ used in~\eqref{eq:final_x_control} and \eqref{eq:final_y_control} was computed from~\eqref{eq:flexible_bending}.
\begin{figure}[htbp]
    \centering
    \includegraphics[keepaspectratio,width=0.45\columnwidth]{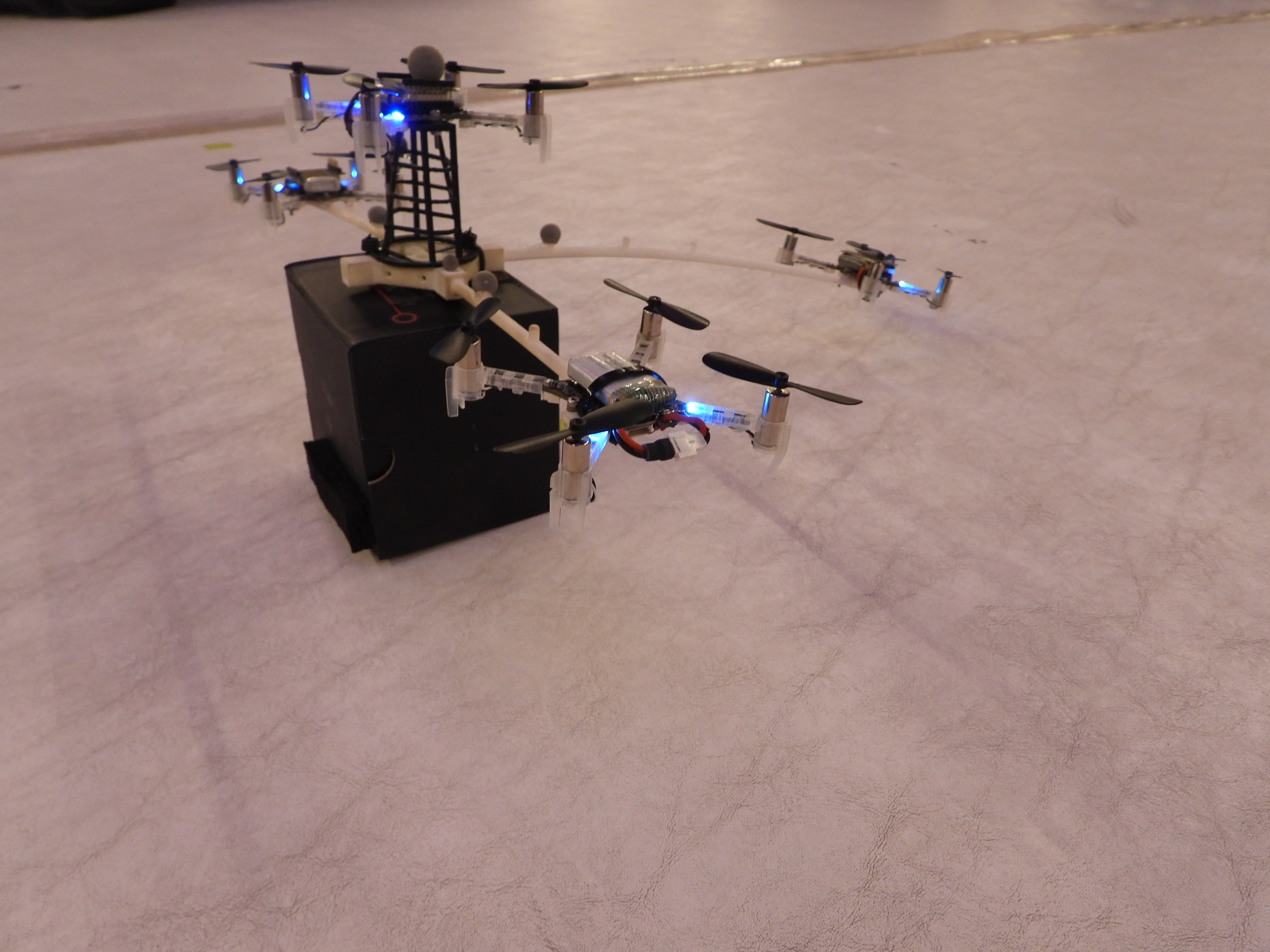}
    \includegraphics[keepaspectratio,width=0.45\columnwidth]{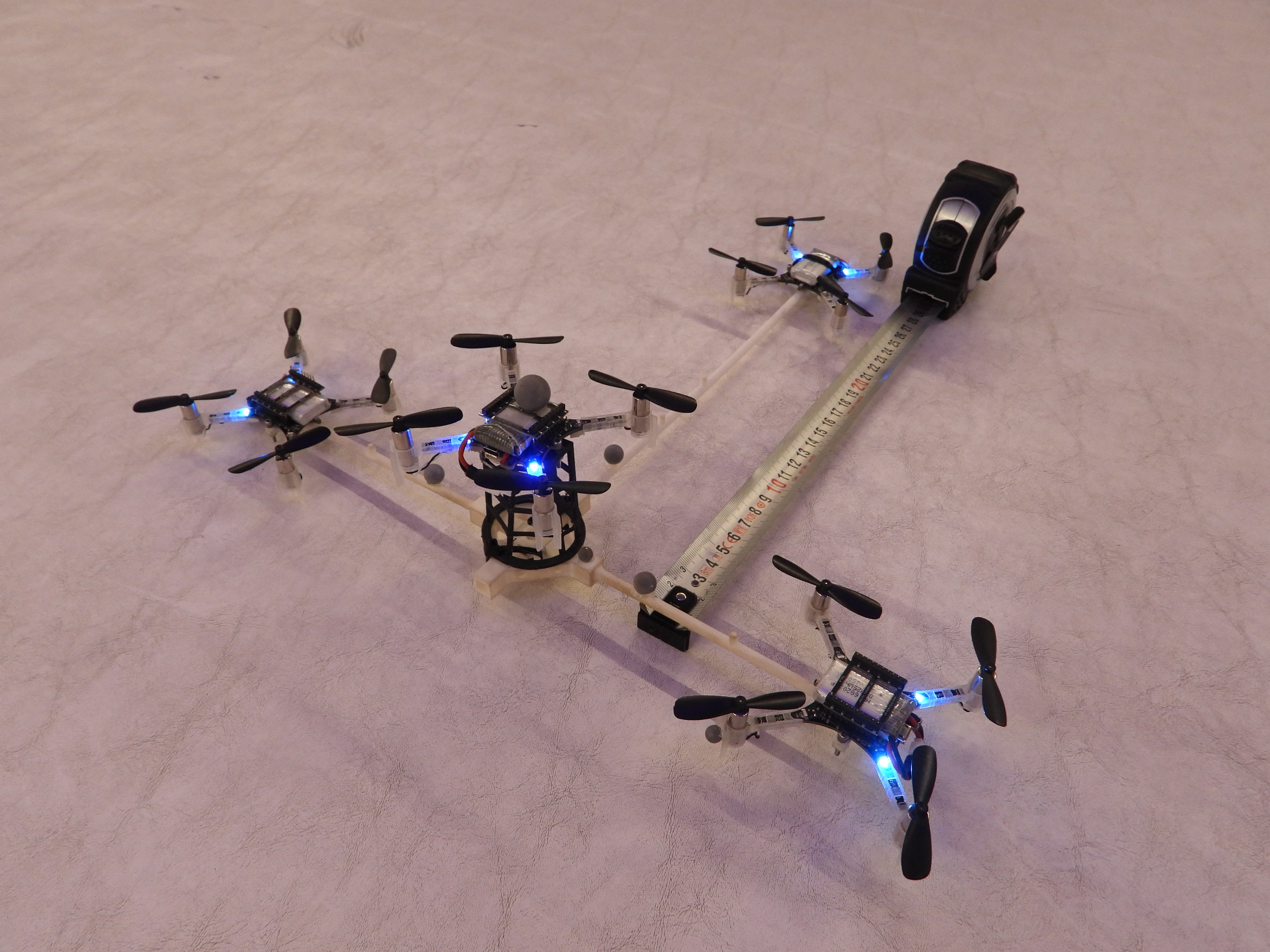}
    \caption{Flexible T-copter ($n=4$) configuration.}
    \label{fig:t-copter-long-bend}
\end{figure}

The static flexibility effects~\eqref{eq:flexible_bending} of the 30-cm clamped free rod was examined for various thrusts. Figure~\ref{fig:bending_comparison} shows the time-history of  $\gamma_3$ angle for the (rod/copter) that had $m_3=$37~g. The applied thrust $T_3$ is shown in a dotted line, while the static~(actual) $\gamma_3$ angle is shown in blue~(red) color. The actual response exhibits significant oscillations due to the unaccounted modes of vibration.
\begin{figure}[htbp]
    \centering
    \includegraphics[keepaspectratio,width=0.6\columnwidth]{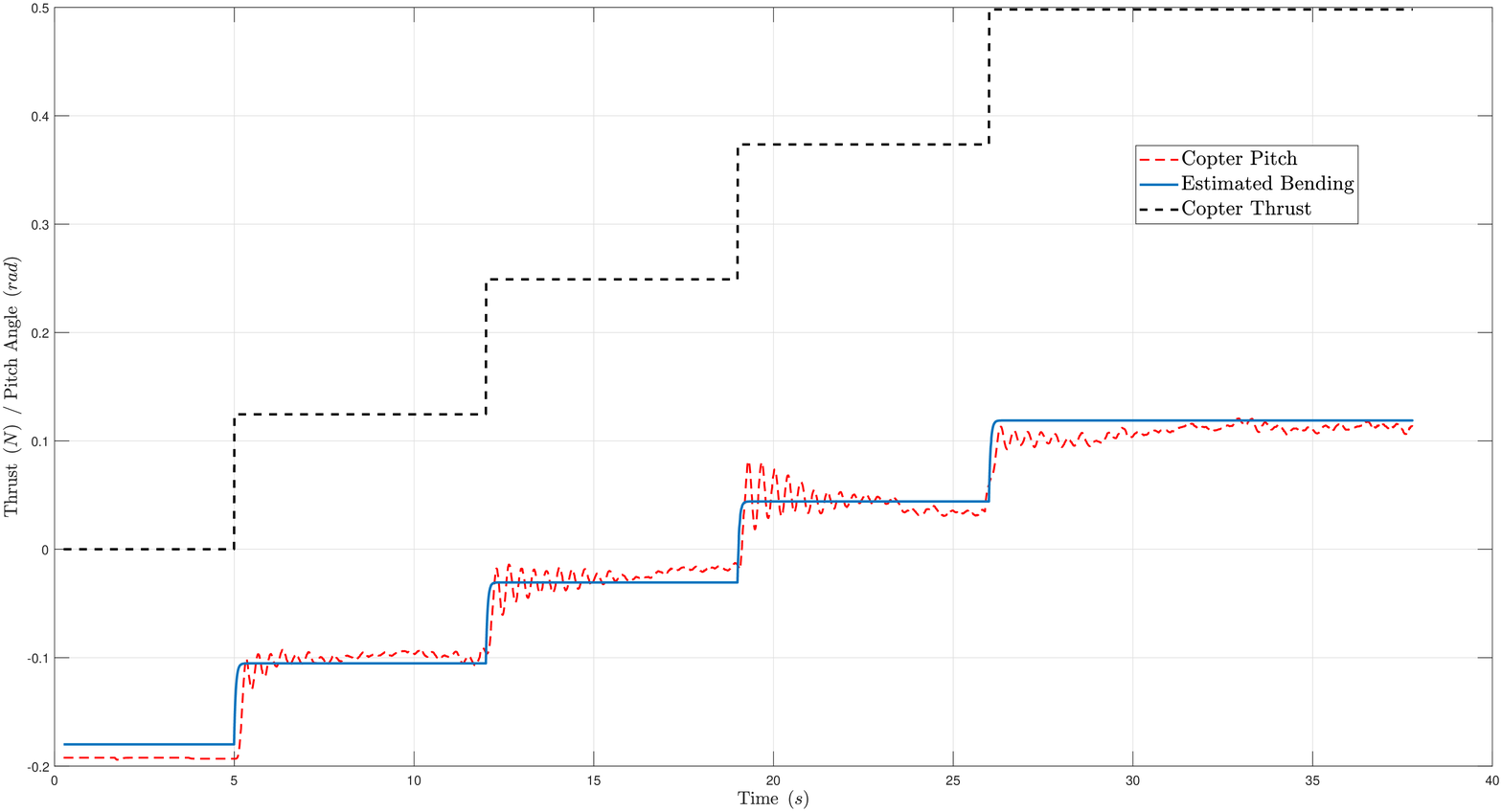}
    \caption{Measured and static bending-angles of elongated rods, under different thrusts.}
    \label{fig:bending_comparison}
\end{figure}

The altitude response of the system appears in Figure~\ref{fig:flex-t-copter}, where there is significant reduction in the system oscillations caused by the system's flexibility.
\begin{figure}[htbp]
    \centering
    \includegraphics[keepaspectratio,width=0.8\columnwidth]{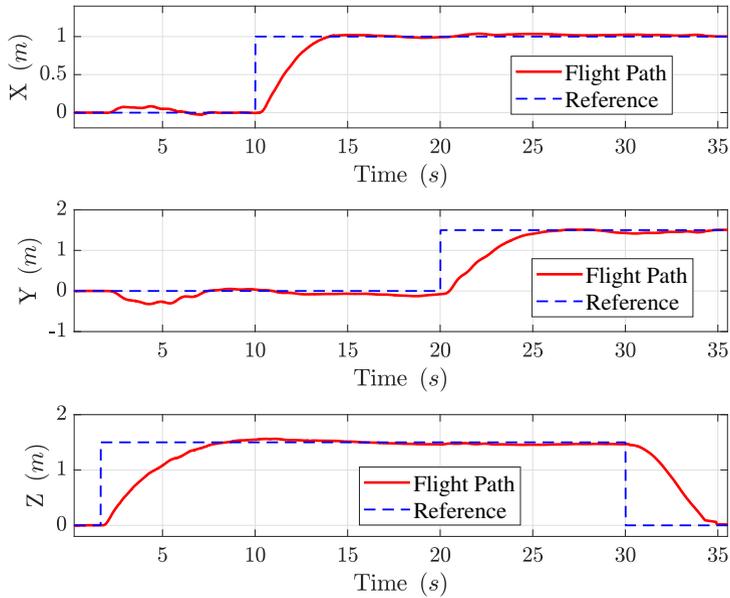}
    \caption{Altitude response of flexible T-copter configuration ($n=4$).}
    \label{fig:flex-t-copter}
\end{figure}

\section{Conclusions}
In this paper, a generalized framework for flying arbitrary modular multi-copter structures was presented. A novel control approach for collaboratively flying such interconnected aerial systems was proposed, relying on the combination of the total thrust produced by each agent, while taking into account the structure's flexibility. The feasibility of the proposed scheme has been experimentally validated using prototype copters and custom designed connecting structure elements.
\clearpage
\section*{Appendix}
\subsection*{Appendix A}
Let $z=z^{\circ}+\Delta z$, and $T=T^{\circ}+\sum_{i=0}^{n-1}\Delta T_i$, where $z^{\circ}$ and $T^{\circ}$ are constants. Then $\dot{z}=\Delta \dot{z}$, and assume the positive definite Lyapunov function $V_1 = \frac{e_z^2}{2}$, where $e_z = z - z^d$ with $z^d$ corresponding to the desired altitude. Then, $
\dot{V}_1 = e_z \dot{e}_z = e_z \left( \Delta\dot{z} - \dot{z}^d \right).$ Let the virtual input $\dot{z}^{*} = \dot{z}^d - K_{z1}e_z$,
and the velocity error term
$
s_z = \Delta\dot{z} - \dot{z}^{*}.
$
Then $\dot{V}_1$ can be rewritten as
$
\dot{V}_1 = e_z \left( s_z - K_{z1}e_z \right).
$
Let the augmented Lyapunov function 
$
V_z = V_1 + \frac{1}{2}s_z^2,
$
then 
\begin{equation}
\dot{V}_z = -K_{z1}e_z^2 + s_z \left( e_z + \Delta\Ddot{z} + K_{z1}(s_z - K_{z1}e_z) \right), \label{eq:V2_dot}
\end{equation}
where it was assumed for simplicity that $\Ddot{z}^d = 0$. 
Given the altitude dynamics~\eqref{eq:deltay}
and the altitude controller~(\ref{eq:final_z_control}) applied to \eqref{eq:V2_dot} results in
$
\dot{V}_z = -K_{z1}e_z^2 - K_{z2}s_z^2 \leq 0.
$
Hence $e_z \rightarrow 0 $, and $s_z \rightarrow 0$, implying that $z \rightarrow z^d$, $\dot{z} \rightarrow \dot{z}^d$, since the system is asymptotically stable under the proposed controller.

Likewise, asymptotic stability of $x$ and $y$ can be proven.

\subsection*{Appendix B}
In the case when there is external payload of unknown mass, the controller needs to be augmented with an online adaptation term for this mass. Assume in~\eqref{eq:final_z_control}, the term $m$ is replaced by its estimate $\hat{m}$, then the Lyapunov function derivative is
\[
\dot{V}_z = -K_{z1}e_z^2 -K_{z2}s_z^2 + \frac{\hat{m}-m}{m} s_z \left( -K_{z1}(s_z - K_{z1}e_z) -e_z - K_{z2}s_z - \sum_{i=0}^{n-1} {\gamma_i} \xi_i^z \right).
\]
Let the positive quantity
\begin{equation}
V_m = \frac{(\hat{m}-m)^2}{2 \sigma m},~\sigma >0\label{eq:mass_lyapunov}
\end{equation}
and the adaptation evolution rule
\begin{equation}
\dot{\hat{m}} = - \sigma s_z \left( -K_{z1}(s_z - K_{z1}e_z) -e_z - K_{z2}s_z - \sum_{i=0}^{n-1} {\gamma_i} \xi_i^z \right). \label{eq:mass_adaptation}
\end{equation}
Then the augmented Lyapunov function $V_z+V_m$ has derivative 
$
\dot{V}_z+\dot{V}_m = -K_{z1}e_z^2 -K_{z2}s_z^2, \label{eq:final_alt_lyap} \leq 0.
$
We should note that there is no guarantee that $\hat{m} \rightarrow m$ but simply that $e_z$ and $s_z$ converge to zero.
\subsection*{Appendix C}
Given the system's attitude dynamics \eqref{eq:flex_rot_exp_2}, the attitude control input $\boldsymbol{\tau}^c$ \eqref{eq:main_control} is computed using adaptive control principles, in order to guarantee the stability of the vehicle's attitude. 

Given desired roll, pitch and yaw angles $(\phi^d, \theta^d, \psi^d)$, let the: a) attitude error vector 
$
\mathbf{e}_{\phi} = \left[ \phi, \theta, \psi\right]^T - \left[ \phi^d, \theta^d, \psi^d\right]^T, 
$
b) ideal angular velocity as 
$
\Omega^* = \Omega^d - K_{\phi} \mathbf{e}_{\phi},
$
where $K_{\phi}$ is a diagonal positive gain matrix, c) velocity error vector $\mathbf{z}_{\phi} = \Omega - \Omega^*$, d) adaptation estimates $\hat{\mathbf{J}}, \hat{\tau}^s$ for inertia matrix and asymmetric torques acting on the structure, and e) error matrix
$
\tilde{E} = I_3 - \mathbf{J}^{-1} \hat{\mathbf{J}},
$
where $I_3$ is the $3\times 3$ identity matrix.

Based on the backstepping principle, a composite Lyapunov function is defined, incorporating attitude errors and errors in unknown estimates
\begin{equation}
V = \frac{1}{2}\mathbf{e}_{\phi}^T \mathbf{e}_{\phi} + \frac{1}{2}\mathbf{z}_{\phi}^T \mathbf{z}_{\phi} + \frac{1}{2} \mbox{tr}\left( \tilde{E}^T \mathbf{J}^T \Lambda^{-1} \tilde{E}\right) + \frac{1}{2 \sigma} (\tau^s - \hat{\tau}^s)^T \mathbf{J}^{-1}(\tau^s - \hat{\tau}^s), \label{eq:total_attitude_lyapunov}
\end{equation}
where $\Lambda$ is a diagonal positive gain matrix and $\sigma$ is a positive constant. 

Substituting \eqref{eq:main_control} to \eqref{eq:total_attitude_lyapunov} and using the adaptation evolutions 
\begin{align}
\dot{\hat{\mathbf{J}}} &= \Lambda^{\top} \mathbf{z}_{\phi} 
\left[ -K_{\phi}(\mathbf{z}_{\phi}-K_{\phi}\mathbf{e}_{\phi})-\mathbf{e}_{\phi}-K_{\omega}\mathbf{z}_{\phi} \right]^{\top}
, \label{eq:inertia_adapt}\\
\dot{\hat{\tau}}^s &= \sigma \mathbf{z}_{\phi}, \label{eq:static_adapt}
\end{align}
then the derivative of the Lyapunov composite function is 
\begin{equation}
\dot{V} = -\mathbf{e}_{\phi}^T K_{\phi} \mathbf{e}_{\phi} - \mathbf{z}_{\phi}^T K_\omega \mathbf{z}_{\phi}\leq 0,
\end{equation}
In this derivation the symmetry and positive definiteness of the inertia matrices is used. 
The inertia matrix estimate $\hat{\mathbf{J}}(0)$ is computed using the application of the parallel axis theorem on the agent masses and $\hat{\tau}^s(0)=0$. Similarly, the same initial estimate for the system inertia matrix is used for the feedforward component of \eqref{eq:main_control}.

\section*{Acknowledgment}
This research was in part performed by using NYUAD's Core Technology Platform Kinesis lab  motion capture system. The authors thank Mr. Nikolaos Giakoumidis and Dr. Oraib Al Ketan for their technical support and insights. 
\bibliography{sn-bibliography}
\end{document}